\documentclass{article}

\usepackage{arxiv}

\usepackage[utf8]{inputenc}
\usepackage[T1]{fontenc}
\usepackage{textcomp}
\usepackage{natbib}
\usepackage{hyperref}
\usepackage{url}
\usepackage{booktabs}
\usepackage{amsfonts}
\usepackage{amsmath}
\usepackage{amssymb}
\usepackage{amsthm}
\usepackage{nicefrac}
\usepackage{microtype}
\usepackage{graphicx}
\usepackage{multirow}
\usepackage{float}
\usepackage{placeins}
\usepackage[table]{xcolor}
\usepackage{tcolorbox}

\graphicspath{{./images/}}

\definecolor{graybg}{HTML}{F0F0F0}

\newtheorem{lemma}{Lemma}
\theoremstyle{remark}

\title{
REIN: Bridging the Gap between Reasoning and Reliability via
Reflection and Abstention Alignment
}

\author{
\textbf{Zhengze Huang\textsuperscript{1,*}} \quad
\textbf{Luyang Yu\textsuperscript{2,*}} \quad
\textbf{Di Hong\textsuperscript{1}} \quad
\textbf{Xinzhe Huang\textsuperscript{1}} \\
\textbf{Wanyu Lin\textsuperscript{3}} \quad
\textbf{Zhixuan Chu\textsuperscript{1}} \quad
\textbf{Zhan Qin\textsuperscript{1,\textdagger}} \quad
\textbf{Tianhang Zheng\textsuperscript{1,\textdagger}} \\
\textsuperscript{1}The State Key Laboratory of Blockchain and Data Security, Zhejiang University \\
\textsuperscript{2}Fudan University \\
\textsuperscript{3}The Hong Kong Polytechnic University \\
\textsuperscript{*}Equal contribution \quad
\textsuperscript{\textdagger}Corresponding author
}

\begin{document}

\maketitle
\begin{abstract}
Large reasoning models (LRMs) are prone to hallucination, which undermines their reliability and poses challenges for safe deployment.
Hallucinations in LRMs arise from two distinct failure sources: reasoning hallucination, where flawed inference steps propagate to an incorrect conclusion, and knowledge hallucination, where the model lacks the requisite factual knowledge to answer the query.
To address reasoning hallucination, we propose REIN, an alignment framework that trains LRMs to produce a structured reasoning sequence, \texttt{<think>} $\rightarrow$ \texttt{<reflection>} $\rightarrow$ \texttt{<answer>}, enabling explicit self-reflection before committing to a final answer.
To address knowledge hallucination, REIN introduces a reward mechanism that encourages explicit abstention (e.g., "I don't know") when none of the sampled reasoning chains yields a correct answer, allowing the model to refrain from unsupported predictions.
Extensive evaluations on mathematical and commonsense reasoning benchmarks show that REIN consistently improves selective accuracy, reduces incorrect-but-self-endorsed responses, and maintains high coverage compared with competitive baselines.
Notably, REIN achieves these gains within a single forward pass, without requiring process supervision, inference-time controllers, external search, or multi-round critiques. Experiments on multiple backbones show that REIN reduces the hallucination proxy by $58\sim72\%$ relative to the base models while maintaining $86\sim91\%$ average coverage, and improves selective accuracy on attempted questions by $6.6\sim14.2\%$.
\end{abstract}

\begin{figure*}[!t]
    \centering
    \includegraphics[width=\textwidth]{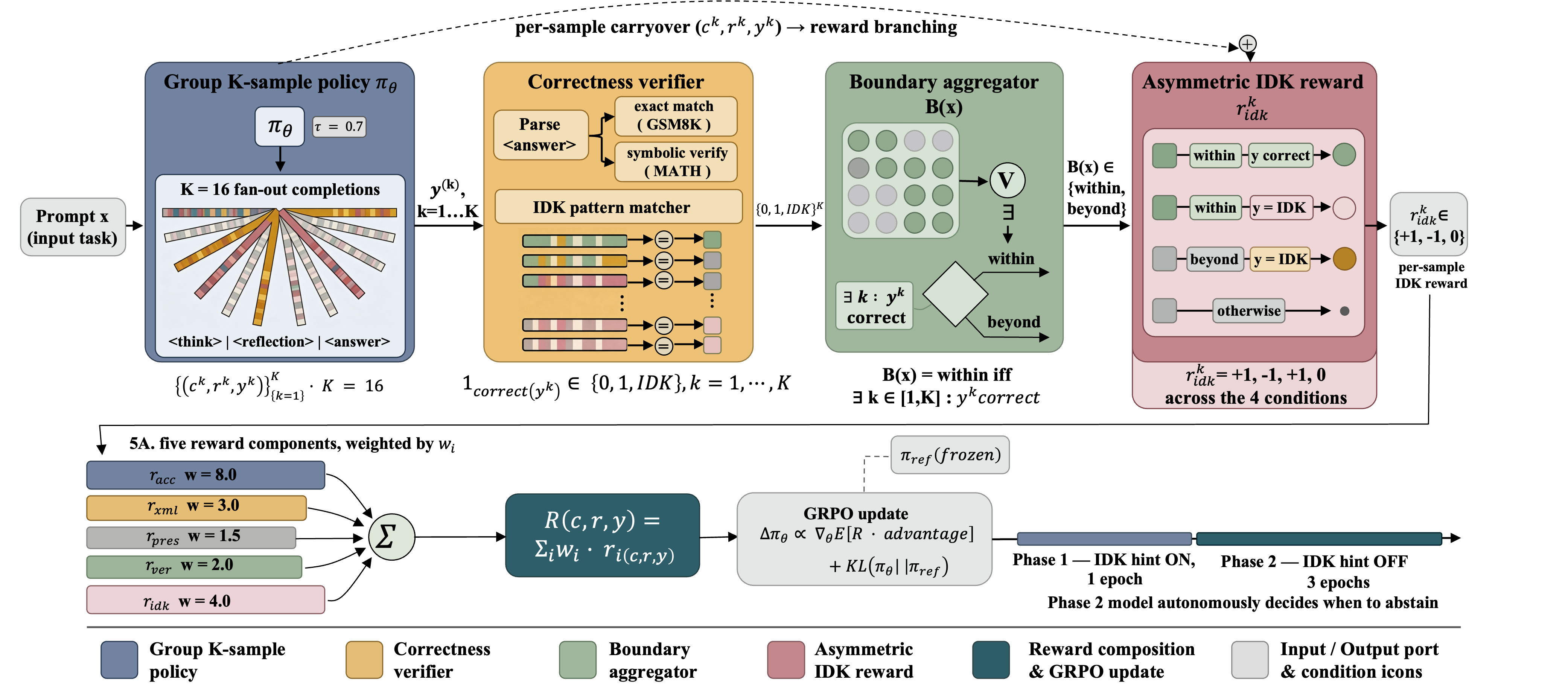}
    \caption{
        Overview of the REIN training pipeline: reflection-augmented
        completions feed into GRPO with reflection and boundary-aware
        abstention rewards.
    }
    \label{fig:method}
\end{figure*}

\section{Introduction}
  \label{sec:intro}

  Large reasoning models (LRMs) such as DeepSeek-R1~\citep{deepseekai2025deepseekr1incentivizingreasoningcapability} and Qwen3~\citep{qwen3technicalreport} have made substantial progress on mathematics, scientific reasoning, and programming~\citep{hendrycks2021, zheng2025scalingphysicalreasoningphysics, chen2025r1codeinterpreterllmsreasoncode} by eliciting explicit chains of thought ~\citep{wei2023chainofthoughtpromptingelicitsreasoning}. Yet explicit reasoning does not eliminate hallucination---it often increases the apparent plausibility of unsupported conclusions, producing
  outputs that seem coherent and persuasive but still contain factual errors~\citep{min2023factscorefinegrainedatomicevaluation} or internally inconsistent
  inference~\citep{cheng2025empoweringllmslogicalreasoning}.

  We identify two primary failure modes of hallucination: reasoning-level hallucination and knowledge-level hallucination.
  % We distinguish two such failure modes. 
  In \emph{reasoning-level hallucination}, the model possesses the requisite knowledge, but a particular sampled trajectory contains a flawed inference step that
   propagates to a wrong answer. This failure mode is self-recoverable, since a different sample or a verified self-check could still yield the correct answer. In \emph{knowledge-level hallucination}, by contrast, the
   model lacks the requisite knowledge under the current policy and sampling budget, and any additional reasoning merely rationalizes an unsupported prediction. Thus, additional reasoning or reflection cannot rescue the answer, and
  the only safe action is to abstain. These two modes demand fundamentally different interventions, and conflating them is a central source of mis-calibration in current alignment recipes.

  Existing approaches address at most one of the two failure modes. Outcome-level preference alignment such as RLHF~\citep{ouyang2022traininglanguagemodelsfollow} and
  DPO~\citep{rafailov2024directpreferenceoptimizationlanguage} optimizes the final response but provides no explicit supervision over the reasoning trajectory, so %reasoning-level errors that survive into the answer are undetected. 
  reasoning-level errors that propagate to the final answer remain undetected.
  Process-supervision and self-correction methods, including step-level process reward models~\citep{uesato2022processfeedback, lightman2024letsverify},
  Self-Refine~\citep{madaan2023selfrefine}, and Reflexion~\citep{shinn2023reflexion}, focus on reliability assessment but presume that the errors are self-recoverable---they have no mechanism to distinguish a fixable inference
   slip from genuinely missing knowledge, and recent evidence shows that without external feedback, models often fail to detect their own errors or even overturn correct answers~\citep{huang2023cannotselfcorrect,
  kamoi2024selfcorrection}. Conversely, abstention and uncertainty-calibration methods such as rejection tuning~\citep{xu2024rejectionimprovesreliabilitytraining}, R-Tuning~\citep{zhang2024rtuning}, and IDK-token
  approaches~\citep{cohen2024idktoken, cheng2024aiassistantsknowdont} encourage refusal under uncertainty but ignore reasoning repair, and therefore over-abstain whenever a recoverable reasoning
  failure is mistaken for an irrecoverable knowledge gap. It remains an open question how to establish a unified, reliability-oriented alignment framework that can decide, within a single completion, whether the current reasoning supports a reliable answer or whether the model should abstain.

  We propose \textbf{REIN}, an alignment framework that teaches an LRM, in a single forward pass, to reflect on its own reasoning before committing to an answer and to abstain 
  when the answer is not adequately supported by the model's internal knowledge.
  REIN trains the model to emit a structured completion \texttt{<think>} $\rightarrow$ \texttt{<reflection>} $\rightarrow$ \texttt{<answer>}, where the \texttt{<reflection>} span is a trainable
  reliability judgment whose stance is aligned, via a centered reflection-veracity reward, with the verified correctness of the answer. To handle knowledge-level failures, REIN further derives a group-level
  boundary indicator from the model's own sampled completions and uses a boundary-aware IDK reward that rewards abstention only when no sample in the group yields a (nearly) correct answer---thereby separating recoverable
  reasoning failures from genuine knowledge gaps without process supervision, external retrieval, or test-time controllers.

  Our contributions are threefold:
  \begin{enumerate}

  \item We propose REIN, a unified single‑pass alignment framework that structures generation as \texttt{<think>} $\rightarrow$ \texttt{<reflection>} $\rightarrow$ \texttt{<answer>}, enabling LRMs to decide within one forward pass whether the current reasoning should be trusted for finalization or whether the model should abstain, without external controllers or iterative loops.
  \item We design a centered reflection‑veracity reward to align the reflection span with the correctness of the pre-reflection draft conclusion, along with a boundary‑aware IDK reward to incentivize abstention when the model lacks sufficient knowledge.
  \item Extensive experiments across multiple backbones and reasoning benchmarks demonstrate that REIN reduces the hallucination proxy by 58–72\% while maintaining 86–91\% average coverage, and improves accuracy on attempted questions by 6.6–14.2\%.

  \end{enumerate}

\section{Related Work}
  \label{sec:related_work}

  \paragraph{Self-correction and reflection.}
  Self-correction methods push the LLM to self-critique and revise its own output. Self-Refine is an inference-time framework that iteratively alternates between model-generated feedback and output revision~\citep{madaan2023selfrefine}. SCoRe instead learns multi-turn self-correction from on-policy traces through reinforcement learning~\citep{kumar2025score}. Evidence on intrinsic self-correction nevertheless shows that self-generated feedback does not consistently correct errors and may overturn correct answers when no reliable external signal is available~\citep{huang2023cannotselfcorrect, kamoi2024selfcorrection, li2024hindsight}. REIN assigns reflection a different role: the \texttt{<reflection>} span assesses whether the preceding reasoning trajectory supports a reliable draft conclusion. Its judgment is aligned with the verified correctness of that pre-reflection conclusion and guides finalization within the same structured completion, without requiring a separate critique--revision loop.

\paragraph{Outcome and process alignment.}
  Preference-based RL such as RLHF~\citep{ouyang2022traininglanguagemodelsfollow}, Safe RLHF~\citep{dai2023saferlhfsafereinforcement}, and DPO~\citep{rafailov2024directpreferenceoptimizationlanguage} aligns model behavior to outcome-level preferences but provides no supervision on the reasoning trajectory itself. Rejection-sampling fine-tuning and verifier-driven RL extend this to verifiable-answer domains via automatic answer checking~\citep{yuan2023rftmath, shao2024deepseekmath, deepseekai2025deepseekr1incentivizingreasoningcapability}, while process-reward models supply step-level signals at the cost of dense human or LLM annotations~\citep{uesato2022processfeedback, lightman2024letsverify, huang2025dualbreach}. Unlike outcome-only preference RL, which is silent about the reasoning trajectory and process-supervised reward models, which require fine-grained step labels and still cannot decide when to abstain. REIN keeps the GRPO outer loop but redesigns the reward family around \emph{reliability}: a reflection-veracity reward aligns the model's self-assessment with answer correctness, and a boundary-aware IDK reward injects an abstention signal, all without step-level supervision.

\paragraph{Uncertainty and abstention.}
R-Tuning and IDK-token methods train models to abstain under uncertainty~\citep{zhang2024rtuning,cohen2024idktoken, cheng2024aiassistantsknowdont}. TruthRL instead uses a ternary RL objective that distinguishes correct answers, hallucinations, and abstentions~\citep{wei2026truthrlincentivizingtruthfulllms}. REIN differs by separately aligning a pre-finalization reflection with the correctness of the draft conclusion implied by the reasoning trajectory, and the final answer-or-abstain decision with a sampled knowledge boundary. KnowRL introduces fine-grained factual supervision over atomic claims in the reasoning trace~\citep{ren2026knowrlexploringknowledgeablereinforcement, wu2026datashield}, whereas REIN requires only task-answer verification. BAPO studies boundary-aware rewards for agentic search~\citep{liu2026bapoboundaryawarepolicyoptimization, zhao2025topologybehavioralsemanticsenhancing}, while BARREL promotes concise, boundary-aware reasoning to mitigate overthinking~\citep{yang2026barrelboundaryawarereasoningfactual, xiu2025dynamic}. In contrast, REIN jointly aligns reflection and abstention within a single structured completion.

\section{REIN Framework}
  \label{sec:method}

\subsection{Framework Setup}
\label{sec:method_problem_setup}

Given a prompt $x\in\mathcal{X}$, the policy $\pi_\theta$ generates a structured completion $z$:
\begin{equation}
    z \sim \pi_\theta(\cdot \mid x).
    \label{eq:structured_completion}
\end{equation}
Each completion follows the order
\begin{equation}
    \text{\texttt{<think>}}
    \rightarrow
    \text{\texttt{<reflection>}}
    \rightarrow
    \text{\texttt{<answer>}}.
    \label{eq:output_structure}
\end{equation}
A deterministic parser extracts the three spans from $z$:
\begin{equation}
    (c,r,y_1):=\mathcal{P}(z).
    \label{eq:parse}
\end{equation}

The reasoning trajectory may imply a preliminary conclusion before reflection.
We denote the draft conclusion extracted from $c$ as
\begin{equation}
    y_0:=\mathcal{D}_x(c),
    \label{eq:draft_answer}
\end{equation}
where $\mathcal{D}_x$ is a task-specific deterministic answer extractor.

Here, $c$ is the pre-reflection reasoning trajectory, $y_0$ is the draft conclusion implied by that trajectory, $r$ is the model's reliability assessment conditioned on the preceding trajectory, and $y_1$ is the final answer generated after reflection. The reflection assesses whether the draft answer $y_0$ implied by the preceding reasoning is reliable. The final answer $y_1$ is either a substantive task answer in $\mathcal{Y}$ or the canonical abstention response \texttt{IDK}. Because $r$ is generated after the reasoning trajectory but before $y_1$,it influences the subsequent finalization decision.

For any candidate answer $y$, we define
\begin{equation}
T(x,y):=
\begin{cases}
1, & \text{if } y\neq\mathrm{IDK}
     \text{ and }\mathcal{V}_x(y,y^\star(x))=1,\\
0, & \text{otherwise}.
\end{cases}
\label{eq:correctness_indicator}
\end{equation}

The draft answer $y_0$ and the final answer $y_1$ are verified for different purposes. The reflection-veracity reward evaluates whether the reflection judgment agrees with $T(x,y_0)$, whereas the answer-correctness and boundary-aware abstention rewards evaluate $T(x,y_1)$. Answers mentioned elsewhere in the reflection span are not treated as task answers. A final answer $y_1=\texttt{IDK}$ is not counted as a correct task answer and is handled separately by the abstention reward. Malformed structured outputs are handled by the format rewards described in Section~\ref{sec:method_optimization}.

Following the two failure modes described in the Introduction, REIN addresses reasoning hallucination and knowledge hallucination. For reasoning hallucination, reflection alignment trains the model to judge whether the draft answer $y_0$ formed by the current reasoning is reliable. This reduces false endorsement, where an incorrect draft answer $y_0$ is nevertheless judged reliable. For knowledge hallucination, REIN estimates a sampled knowledge boundary from a group of on-policy completions. If the group contains no verified-correct final answer $y_1$, the model is encouraged to abstain instead of producing an unsupported final answer. Together, the two signals guide the model's final answer-or-abstain decision in $y_1$.

\subsection{Reflection Alignment for Reasoning Hallucination}
\label{sec:method_reflection}

Before generating the final answer, REIN uses the \texttt{<reflection>} span to assess whether the preceding reasoning trajectory supports a reliable draft conclusion $y_0$. After the completion is generated, REIN verifies the draft conclusion $y_0$ and rewards the reflection when its reliability judgment matches $T(x,y_0)$. The post-reflection final answer $y_1$ is not used as the target of the reflection-veracity reward. This distinction prevents a correct assessment of an unreliable draft trajectory from being penalized when the model subsequently revises its final answer.

A fixed parser maps the reflection $r$ to
\[
s(r)\in
\{\texttt{correct},\texttt{wrong},\texttt{uncertain},\bot\},
\]
where $\bot$ denotes a missing, invalid, or ambiguous stance. For a valid stance, we define the binary reflection judgment as
\begin{equation}
J(r):=
\begin{cases}
1, & s(r)=\texttt{correct},\\
0, & s(r)\in\{\texttt{wrong},\texttt{uncertain}\}.
\end{cases}
\label{eq:reflection_indicator}
\end{equation}
Here, $J(r)=1$ indicates that the model considers the conclusion implied by the preceding reasoning trajectory reliable, while $J(r)=0$ indicates that it considers the trajectory unreliable.

The main failure targeted by reflection alignment is false endorsement:
\begin{equation}
\mathrm{FE}(x,r,y_0)
:=
\mathbb{1}[y_0\in\mathcal{Y}]
\bigl(1-T(x,y_0)\bigr)J(r).
\label{eq:false_endorsement}
\end{equation}
False endorsement occurs when the draft conclusion implied by the reasoning trajectory is incorrect but the reflection nevertheless assigns it a reliable stance. Such over-confidence may cause the model to preserve an unreliable trajectory during finalization.

The opposite mismatch is false rejection:
\begin{equation}
\mathrm{FR}(x,r,y_0)
:=
T(x,y_0)\bigl(1-J(r)\bigr).
\label{eq:false_rejection}
\end{equation}
False rejection occurs when a reliable draft conclusion is judged unreliable. Although the final answer $y_1$ may still be correct, this mismatch can lead to unnecessary revision or abstention.

The reflection-veracity reward is defined as
\begin{equation}
r_{\mathrm{vrcty}}(x,r,y_0):=
\begin{cases}
+1, & s(r)\neq\bot \ \wedge\ J(r)=T(x,y_0),\\
-1, & s(r)\neq\bot \ \wedge\ J(r)\neq T(x,y_0),\\
0, & s(r)=\bot.
\end{cases}
\label{eq:vrcty}
\end{equation}
Invalid or missing reflection stances receive no veracity reward and are handled by the format rewards described in Section~\ref{sec:method_optimization}.

Reflection alignment works at the completion level. It aligns a reflection judgment with the verified reliability of the draft conclusion implied by the preceding reasoning trajectory. It does not perform step-level verification or identify which reasoning step causes an error. The correctness of the final answer $y_1$ is supervised separately by the answer-correctness reward, while the final answer-or-abstain decision is further guided by the group-level signal introduced next.

\subsection{Boundary-Aware Abstention Alignment for Knowledge Hallucination}
\label{sec:method_abstention}

Reflection alignment assesses whether a single reasoning trajectory supports a reliable draft conclusion $y_0$, but it cannot determine whether the current policy can produce a verified final answer $y_1$ for the prompt. REIN therefore uses a group of on-policy completions to obtain an empirical estimate of the current policy's knowledge boundary.

For each prompt $x$, the current policy samples $K$ completions:
\begin{equation}
z^{(k)} \sim \pi_\theta(\cdot \mid x),
\qquad k=1,\ldots,K.
\label{eq:group_sampling}
\end{equation}

Applying $\mathcal{P}$ and $\mathcal{D}_x$ to each sampled completion yields
\begin{equation}
\mathcal{G}_K(x)
:=
\left(
(c^{(k)}, y_0^{(k)}, r^{(k)}, y_1^{(k)})
\right)_{k=1}^{K},
\label{eq:sampled_group}
\end{equation}
where
$(c^{(k)}, r^{(k)}, y_1^{(k)})=\mathcal{P}(z^{(k)})$
and
$y_0^{(k)}=\mathcal{D}_x(c^{(k)})$.

We mark the prompt as \texttt{within} the sampled boundary if at least one completion produces a verified-correct final answer $y_1^{(k)}$. 
Otherwise, the prompt is marked as \texttt{beyond}:
\begin{equation}
\widehat{B}_K(x)
:=
\begin{cases}
\texttt{within},
& \exists\, k \in \{1,\ldots K\}:
  T(x,y_1^{(k)})=1\\
\texttt{beyond},
& \text{otherwise}.
\end{cases}
\label{eq:sampled_boundary}
\end{equation}

The estimated boundary depends on the current policy, the group size $K$, and the decoding configuration. It is therefore an empirical, policy-relative boundary rather than an oracle label of whether the relevant knowledge is absent from the model parameters. Larger groups reduce finite-sampling uncertainty but increase rollout and verification costs. We use $K=16$ as a practical trade-off between estimation stability and computational cost. Appendix~G provides the probabilistic interpretation, finite-sample confidence analysis, and sensitivity results for different rollout budgets.

REIN uses this boundary to supervise the final answer-or-abstain decision in $y_1$. When $\widehat{B}_K(x)=\texttt{within}$, the sampled group contains at least one verified-correct final answer $y_1^{(k)}$. REIN therefore rewards verified-correct final answers and penalizes unnecessary abstention in $y_1$. When $\widehat{B}_K(x)=\texttt{beyond}$, REIN rewards $y_1^{(k)}=\texttt{IDK}$ instead of an unverified substantive final answer.

For completion $k$, the boundary-aware abstention reward is
\begin{equation}
r_{\mathrm{idk}}^{(k)}
:=
\begin{cases}
+1,
& \widehat{B}_K(x)=\texttt{within}
  \ \wedge\ T(x,y_1^{(k)})=1,\\
-1,
& \widehat{B}_K(x)=\texttt{within}
  \ \wedge\ y_1^{(k)}=\texttt{IDK},\\
+1,
& \widehat{B}_K(x)=\texttt{beyond}
  \ \wedge\ y_1^{(k)}=\texttt{IDK},\\
0,
& \text{otherwise}.
\end{cases}
\label{eq:idk}
\end{equation}

We detect abstention from the parsed \texttt{<answer>} span $y_1$ using a fixed set of IDK patterns. This reward provides a prompt-level empirical knowledge-boundary signal. A single failed completion does not by itself justify abstention when another rollout for the same prompt has already produced a verified-correct final answer. Incorrect substantive final answers receive no positive reward from $r_{\mathrm{idk}}$ and are evaluated separately by $r_{\mathrm{acc}}$. The reflection-veracity reward separately supervises whether the reflection $r$ correctly assesses the reliability of the corresponding reasoning trajectory and its implied draft conclusion $y_0$.

\begin{table}[!t]
\centering
\small
\renewcommand{\arraystretch}{1.1}
\setlength{\tabcolsep}{0.9pt}

\resizebox{\textwidth}{!}{%
\begin{tabular}{l *{4}{ccc}}
\toprule

\multirow{3}{*}{\textbf{Method}}
& \multicolumn{6}{c}{\textbf{Mathematical Reasoning}}
& \multicolumn{6}{c}{\textbf{Commonsense Reasoning}} \\

\cmidrule(lr){2-7}
\cmidrule(lr){8-13}

& \multicolumn{3}{c}{\textbf{GSM8K}}
& \multicolumn{3}{c}{\textbf{MATH}}
& \multicolumn{3}{c}{\textbf{StrategyQA}}
& \multicolumn{3}{c}{\textbf{ARC-Challenge}} \\

\cmidrule(lr){2-4}
\cmidrule(lr){5-7}
\cmidrule(lr){8-10}
\cmidrule(lr){11-13}

& Sel.Acc.$\uparrow$
& H-Proxy$\downarrow$
& Eff.Acc.$\uparrow$
& Sel.Acc.$\uparrow$
& H-Proxy$\downarrow$
& Eff.Acc.$\uparrow$
& Sel.Acc.$\uparrow$
& H-Proxy$\downarrow$
& Eff.Acc.$\uparrow$
& Sel.Acc.$\uparrow$
& H-Proxy$\downarrow$
& Eff.Acc.$\uparrow$ \\

\midrule

% ============================================================
% Qwen2.5-7B
% ============================================================
\rowcolor{graybg}
\multicolumn{13}{l}{\textbf{Backbone: Qwen2.5-7B}} \\

Base
&79.1&41.0&79.1
&41.2&45.0&41.2
&64.3&38.0&64.3
&74.8&36.0&74.8\\

Instruct
&83.4&33.0&83.4
&49.8&39.0&49.8
&69.5&31.0&69.5
&80.3&30.0&80.3\\

IDK Prompting
&86.8&22.5&78.9$_{(90.9)}$
&54.1&31.1&38.7$_{(71.5)}$
&73.4&22.0&60.2$_{(82.0)}$
&84.2&20.8&76.7$_{(91.1)}$\\

R-Tuning
&84.2&13.5&68.8$_{(81.7)}$
&51.2&27.5&40.4$_{(79.4)}$
&64.3&12.4&57.5$_{(89.4)}$
&86.1&17.2&81.9$_{(95.2)}$\\

TruthRL
&86.9&17.9&80.6$_{(92.7)}$
&57.3&42.7&48.7$_{(85.0)}$
&67.7&26.9&62.7$_{(92.6)}$
&83.7&10.8&77.5$_{(92.6)}$\\

Reflection-only
&85.2&18.9&\textbf{85.2}
&51.2&22.1&\textbf{51.2}
&70.3&17.5&\textbf{70.3}
&82.1&16.2&\textbf{82.1}\\

\textbf{REIN (Ours)}
&\textbf{90.6}&\textbf{11.0}&85.0$_{(93.8)}$
&\textbf{61.0}&\textbf{15.0}&50.9$_{(83.5)}$
&\textbf{76.2}&\textbf{10.0}&69.2$_{(90.8)}$
&\textbf{86.5}&\textbf{9.0}&80.9$_{(93.5)}$\\

\midrule

% ============================================================
% Llama-3.1-8B
% ============================================================
\rowcolor{graybg}
\multicolumn{13}{l}{\textbf{Backbone: Llama-3.1-8B}} \\

Base
&76.8&44.0&76.8
&38.5&48.0&38.5
&66.7&37.0&66.7
&76.2&35.0&76.2\\

Instruct
&81.6&35.0&81.6
&47.3&42.0&47.3
&71.5&30.0&71.5
&82.8&28.0&82.8\\

IDK Prompting
&85.0&25.2&74.2$_{(87.3)}$
&51.8&33.9&36.8$_{(71.0)}$
&74.8&21.1&67.4$_{(90.1)}$
&86.2&20.5&76.2$_{(88.4)}$\\

R-Tuning
&86.5&21.4&77.9$_{(90.0)}$
&52.9&28.9&44.8$_{(84.6)}$
&73.8&23.1&69.0$_{(93.5)}$
&79.6&18.4&77.7$_{(97.7)}$\\

TruthRL
&77.8&22.1&77.3$_{(99.3)}$
&53.2&50.6&47.3$_{(89.0)}$
&73.7&19.2&67.8$_{(92.9)}$
&82.5&17.2&81.3$_{(98.6)}$\\

Reflection-only
&83.5&21.2&\textbf{83.5}
&49.2&26.0&\textbf{49.2}
&72.9&17.7&\textbf{72.9}
&86.2&15.9&\textbf{86.2}\\

\textbf{REIN (Ours)}
&\textbf{89.2}&\textbf{12.0}&83.0$_{(93.0)}$
&\textbf{59.5}&\textbf{18.0}&48.5$_{(81.5)}$
&\textbf{78.2}&\textbf{11.0}&72.3$_{(92.5)}$
&\textbf{88.0}&\textbf{10.0}&83.6$_{(95.0)}$\\

\bottomrule
\end{tabular}%
}

\caption{Comparison of reliability alignment methods. 
Sel.Acc. denotes accuracy among answered examples, and H-Proxy measures the fraction of incorrect answers incorrectly judged as reliable. 
Eff.Acc. denotes overall accuracy over all examples. 
Coverage values below 100\% are reported as subscripts of Eff.Acc., while omitted entries indicate full coverage. 
Higher Sel.Acc. and Eff.Acc. are preferred, whereas lower H-Proxy is better.}

\label{tab:main_results}
\end{table}

\subsection{Joint Reward and GRPO Optimization}
\label{sec:method_optimization}

REIN first applies a short supervised fine-tuning stage to initialize the structured output format and the reflection labels used to assess the reliability of the pre-reflection reasoning trajectory and its implied draft conclusion $y_0$. The resulting policy is then optimized with GRPO over the on-policy completion groups defined in Section~\ref{sec:method_abstention}.

For each completion $k$, the total reward combines final-answer correctness, format control, trajectory-grounded reflection alignment, and final-answer abstention alignment:
\begin{equation}
\begin{aligned}
R^{(k)}(x)
={}&
w_{\mathrm{acc}}r_{\mathrm{acc}}^{(k)}
\\
&+
w_{\mathrm{xml}}r_{\mathrm{xml}}^{(k)}
+
w_{\mathrm{pres}}r_{\mathrm{pres}}^{(k)}
\\
&+
w_{\mathrm{vrcty}}r_{\mathrm{vrcty}}^{(k)}
+
w_{\mathrm{idk}}r_{\mathrm{idk}}^{(k)}.
\end{aligned}
\label{eq:composite}
\end{equation}

For completion $k$, the answer-correctness reward $r_{\mathrm{acc}}^{(k)}$ evaluates the post-reflection final answer $y_1^{(k)}$ and rewards it when it is verified correct. The format rewards $r_{\mathrm{xml}}^{(k)}$ and $r_{\mathrm{pres}}^{(k)}$ enforce the required tag structure and the presence of a non-empty reflection span, respectively. The reflection-veracity reward $r_{\mathrm{vrcty}}^{(k)}$ aligns the reflection judgment $r^{(k)}$ with the verified correctness of the corresponding draft conclusion $y_0^{(k)}$ implied by the preceding reasoning trajectory. The boundary-aware reward $r_{\mathrm{idk}}^{(k)}$ aligns the final answer-or-abstain action expressed in $y_1^{(k)}$ with the group-level sampled knowledge boundary.

Within GRPO, the rewards of the $K$ completions sampled for the same prompt are normalized to obtain relative advantages:
\begin{equation}
\widehat{A}^{(k)}
=
\frac{
R^{(k)}
-
\operatorname{mean}_{j}R^{(j)}
}{
\operatorname{std}_{j}R^{(j)}
+
\epsilon
}.
\label{eq:group_advantage}
\end{equation}

REIN then applies the standard GRPO objective with KL regularization relative to the reference policy. While $r_{\mathrm{vrcty}}^{(k)}$ supervises whether the reflection correctly assesses the reliability of the pre-reflection reasoning trajectory, $r_{\mathrm{acc}}^{(k)}$ and $r_{\mathrm{idk}}^{(k)}$ supervise final-answer correctness and abstention in $y_1^{(k)}$, respectively. All reward components contribute to the same scalar reward and update the same policy. Therefore, trajectory reliability assessment and final answer-or-abstain decisions are optimized jointly.

REIN further uses a two-phase GRPO curriculum to establish abstention behavior. In the first phase, the training prompt explicitly presents $y_1=\texttt{IDK}$ as an available final action. In the second phase, the explicit IDK cue is removed while the boundary-aware reward is retained. The policy must then determine whether $y_1$ should contain a substantive answer or \texttt{IDK}, without a direct IDK instruction.

At inference time, REIN generates a single structured completion without reference answers, task verifiers, group sampling, or multi-round critique--revision loops. Within this single completion, the trained policy forms a draft conclusion $y_0$ from the reasoning trajectory, assesses its reliability through $r$, and produces a final answer $y_1$ that is either substantive or \texttt{IDK}.

\section{Experimental Setup}

\paragraph{Datasets and Evaluation Metrics.}
We evaluate all methods on four benchmarks: GSM8K~\citep{cobbe2021gsm8k}
and MATH-500~\citep{hendrycks2021math} for mathematical reasoning, and
StrategyQA~\citep{geva2021strategyqa} and
ARC-Challenge~\citep{clark2018think} for commonsense reasoning.
GRPO optimization uses only the GSM8K training set; results on the other
three benchmarks therefore measure the cross-task transfer of the learned
reliability behavior.

We report four primary metrics: Selective Accuracy
(\textbf{Sel.Acc.}), Hallucination Proxy
(\textbf{H-Proxy}), Coverage
(\textbf{Cov.}), and Effective Accuracy
(\textbf{Eff.Acc.}).

Sel.Acc.\ measures correctness among substantive final answers. H-Proxy measures the proportion of incorrect pre-reflection draft conclusions \(y_0\) that are nevertheless positively endorsed by the corresponding reflection \(r\). Specifically, it estimates \(\Pr[J(r)=1 \mid T(x,y_0)=0]\), thereby measuring draft-level false endorsement rather than the overall incidence of incorrect final answers. Cov.\ is the proportion of evaluation examples for which the model produces a substantive final answer, while Eff.Acc.\ is the proportion of all examples that are answered correctly. Higher Sel.Acc., Cov., and Eff.Acc.\ are better, whereas lower H-Proxy is better. Detailed definitions of these metrics are provided in Appendix~A.

\paragraph{Models and Baselines.}
We evaluate REIN on two backbone models: Qwen2.5-7B~\citep{qwen25technicalreport} and Llama-3.1-8B~\citep{dubey2024llama3herd}.
These backbones represent different model families and allow us to test whether REIN generalizes beyond a single architecture. We use Base and Instruct as checkpoint baselines, IDK Prompting as a training-free baseline, Reflection-only as a reflection-alignment baseline, and R-Tuning~\citep{cohen2024idktoken} as an explicit uncertainty-modeling baseline. We also use TruthRL~\citep{wei2026truthrlincentivizingtruthfulllms} as a reinforcement-learning baseline for factuality and abstention.
We reproduce all training-based baselines using their original objectives and adapt them to our experimental setting. All methods, including REIN, are evaluated using the same test splits and a unified decoding and evaluation protocol.For H-Proxy evaluation, every method is prompted to generate the same \texttt{<think>} \(\rightarrow\) \texttt{<reflection>} \(\rightarrow\) \texttt{<answer>} structure within a single autoregressive completion. The reflection therefore serves as an endogenous pre-finalization judgment of the model's own draft conclusion \(y_0\); no additional post-hoc judging call is used.

We additionally evaluate REIN on Mistral-7B-v0.3~\citep{mistralai2024mistral7bv03} and DeepSeek-R1-Distill-8B~\citep{deepseekai2025deepseekr1incentivizingreasoningcapability}, and observe consistent improvements. The corresponding results and further
implementation details are provided in Appendix~C.

\paragraph{Training Details.}
Unless otherwise stated, all backbones use the same training configuration. We first perform a short SFT stage to initialize the structured output format and reflection behavior, followed by GRPO training on GSM8K. For each prompt, GRPO samples $K=16$ completions with a maximum length of 2048 and temperature $0.7$. We use LoRA with $r=16$ and $\alpha=32$, and set the KL coefficient to $0.05$. The weights for answer correctness, XML structure, reflection presence, reflection veracity, and IDK alignment are $8.0$, $3.0$, $1.5$, $2.0$, and $4.0$, respectively. Training follows the two-phase IDK curriculum described in Section~\ref{sec:method_optimization}. At evaluation, all models use single-pass greedy decoding without self-consistency or multi-round revision.

\section{Main Results}

\label{sec:main_results}

\paragraph{Overall results.}

Table~\ref{tab:main_results} reports the main comparison on Qwen2.5-7B and Llama-3.1-8B across four benchmarks. Macro-averaged over the eight backbone--benchmark pairs, REIN obtains 78.7\% selective accuracy, 12.0\% H-Proxy, 90.5\% coverage, and 71.7\% effective accuracy. REIN achieves the best selective accuracy and the lowest H-Proxy in every reported cell, indicating that the improvement is not restricted to a particular backbone or task family. Relative to the Base checkpoints, REIN improves selective accuracy by 14.0 \% and effective accuracy by 7.0 \% on average, while reducing H-Proxy by 28.5 \%. The accompanying 9.6-point reduction in coverage is therefore not a simple loss of utility: the model answers fewer low-support cases but still produces more correct answers over the complete evaluation set.

\begin{figure}[!t]
    \centering
    \includegraphics[width=0.75\textwidth]{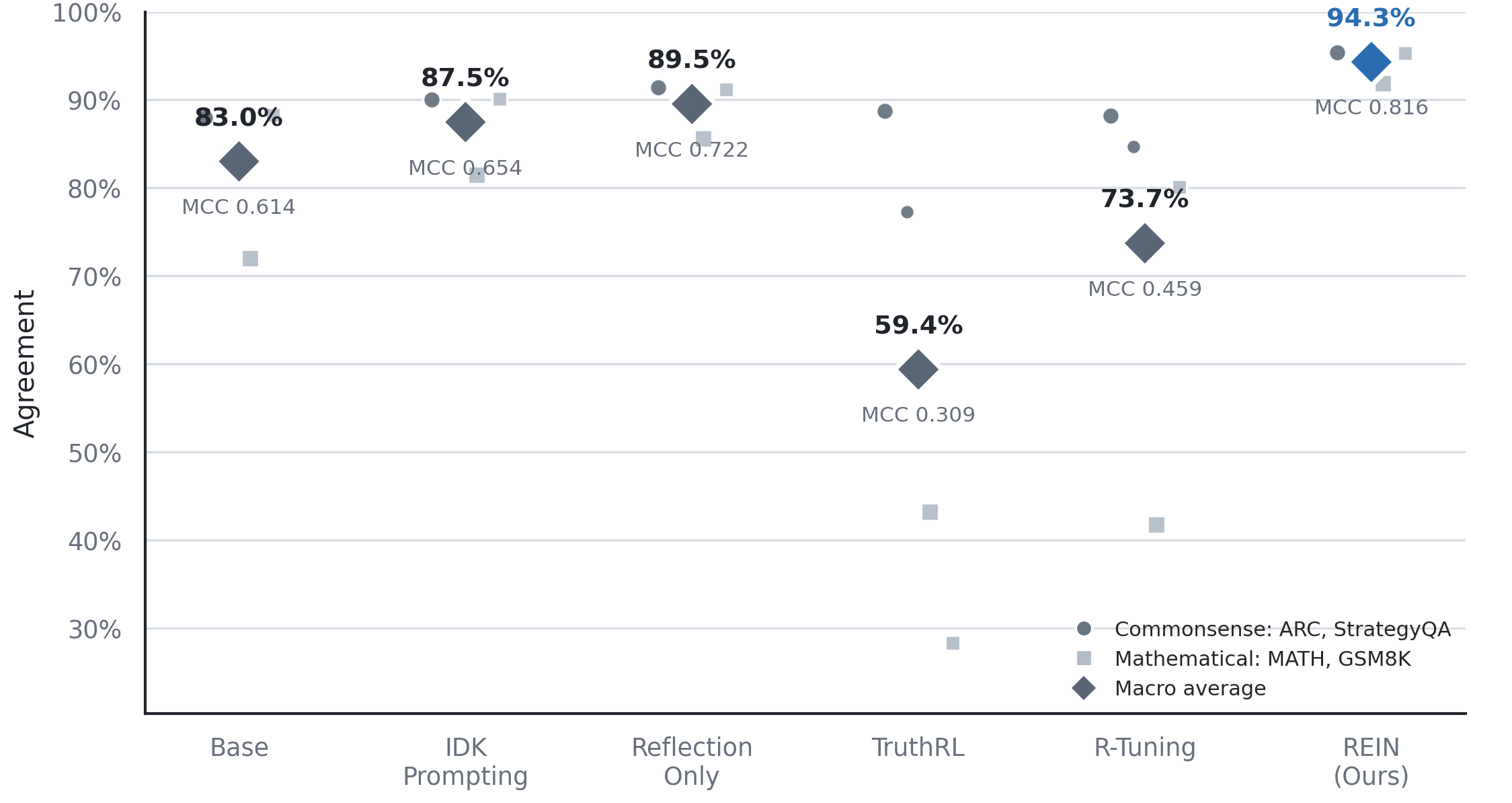}
    \caption{Reliability judgment calibration across methods and backbones. We report reflection–answer agreement and Matthews correlation coefficient (MCC), showing that REIN improves the alignment between self-reflection judgments and actual answer correctness.}
    \label{fig:mcc_fig1}
\end{figure}

\paragraph{Comparison with different methods}
Compared with IDK Prompting, REIN reduces H-Proxy by 12.6 \% while
improving coverage by 6.4 \% and effective accuracy by 8.0 \%.
This comparison is important because both methods permit abstention, but
prompting alone tends to refuse without learning a stable reliability
decision rule. Compared with R-Tuning, REIN improves selective accuracy
by 6.3 \% and effective accuracy by 6.9 \%, with both higher
coverage and lower H-Proxy. Relative to TruthRL, REIN reduces H-Proxy by
13.9 \% and improves effective accuracy by 3.8 \%, at the cost of
a moderate 2.4-point reduction in coverage. These results place REIN in a
more favorable reliability--utility region rather than merely moving
along a refusal-only trade-off.

\paragraph{Behavior across task difficulty.}
The gains are consistent across mathematical and commonsense reasoning,
but their form varies with task difficulty. On MATH, REIN improves
selective accuracy over Base by 19.8 \% for Qwen2.5-7B and 21.0
\% for Llama-3.1-8B, while reducing coverage by 16.5 and 18.5 \%,
respectively. In contrast, the coverage reductions on GSM8K and
ARC-Challenge are substantially smaller. This pattern is consistent with
a difficulty-sensitive policy: on a benchmark with more prompts near the
current policy's capability frontier, REIN abstains more aggressively,
whereas on easier tasks it preserves most of the original answer rate.
Importantly, effective accuracy still increases on MATH for both
backbones, showing that the additional abstentions remove errors faster
than they remove correct answers.

\subsection{Ablation Study}

\begin{table}[!htbp]
\centering
\small
\renewcommand{\arraystretch}{1.08}
\setlength{\tabcolsep}{3pt}

\resizebox{0.5\textwidth}{!}{%
\begin{tabular}{@{}lcccc@{}}
\toprule
\textbf{Configuration}
& \textbf{Sel.Acc.}$\uparrow$
& \textbf{H-Proxy}$\downarrow$
& \textbf{Cov.}$\uparrow$
& \textbf{Eff.Acc.}$\uparrow$ \\
\midrule

Core GRPO
& 84.8 & 28.0 & 99.8 & 84.6 \\

\quad + Reflection
& 85.2 & 19.0 & \textbf{100.0} & \textbf{85.2} \\

\quad + IDK
& 90.2 & 20.0 & 93.5 & 84.3 \\

\quad + Reflection + IDK
& \textbf{90.6} & \textbf{11.0} & 93.8 & 85.0 \\

\bottomrule
\end{tabular}%
}

\caption{
Controlled component ablation on GSM8K with Qwen2.5-7B.
All variants share the same SFT initialization, GRPO configuration,
accuracy reward, and format rewards.
}
\label{tab:core_component_ablation}
\end{table}

Additional analyses on group size, reward components, and IDK behavior are reported in Appendix E.

\paragraph{Effect of Reflection Alignment.}

To investigate how the amount of reflection affects reliability,
we vary the reflection depth $d \in \{1, 2, 3\}$ during training, where $d=1$ uses a single \texttt{<reflection>} block and $d>1$ uses nested reflection-correction cycles.
Table~\ref{tab:depth} reports results on GSM8K.

\begin{table}[!htbp]
\centering
\small
\resizebox{0.5\textwidth}{!}{%
\begin{tabular}{@{}ccccc@{}}

\toprule
Depth $d$ & Sel.Acc$\uparrow$ & H-proxy$\downarrow$ & Cov.$\uparrow$ & Eff.Acc $\uparrow$\\
\midrule
1 & \textbf{90.6} & 0.11 & \textbf{93.8} & \textbf{85.0} \\
2 & 89.1 & 0.08 & 90.3 & 80.4 \\
3 & 88.1 & \textbf{0.06} & 88.4 & 77.9 \\
\bottomrule
\end{tabular}%
}
\caption{Effect of reflection depth $d$ on GSM8K (7B). Deeper reflection further suppresses hallucination but reduces accuracy and coverage due to overcorrection.}
\label{tab:depth}
\end{table}

A single reflection layer ($d=1$) yields the best trade-off between hallucination reduction and reasoning fidelity. Deeper reflection ($d=2, 3$) further suppresses hallucinations (0.11 $\rightarrow$ 0.08 $\rightarrow$ 0.06) but reduces coverage (93.8\% $\rightarrow$ 90.3\% $\rightarrow$ 88.4\%) and effective accuracy (85.0 $\rightarrow$ 80.4 $\rightarrow$ 77.9), suggesting an \emph{overcorrection} phenomenon where the model becomes overly cautious in self-judgment. This result supports our design choice that reflection alignment should remain lightweight.

We compare Base with Reflection-only to examine the overall effect of reflection-structured training.
Reflection-only uses the same structured output format and reflection rewards as REIN, but removes the boundary-aware IDK reward.
Averaged over the four benchmarks, Reflection-only improves selective accuracy over Base for all two backbones and substantially reduces the hallucination proxy.
However, Reflection-only keeps 100\% coverage, showing that reflection-structured training improves reliability assessment but does not teach the model when to abstain.

\paragraph{Effect of Boundary-Aware Abstention.}

Table~\ref{tab:idk_quality} reports IDK precision and recall on all datasets, where ``truly beyond boundary'' is determined by oracle evaluation (the gold answer appears in no sampled completion across 64 independent samples).
\begin{table}[H]
\centering
\small

\resizebox{0.5\textwidth}{!}{%
\begin{tabular}{@{}lccc@{}}
\toprule
Dataset & IDK Prec.$\uparrow$ & IDK Recall$\uparrow$ & IDK Rate \\
\midrule
GSM8K      & 82.6 & 71.3 & 6.2\% \\
MATH       & 78.4 & 65.7 & 17.3\% \\
StrategyQA & 80.1 & 68.5 & 8.5\% \\
ARC-C      & 83.2 & 72.8 & 6.8\% \\
\bottomrule
\end{tabular}%
}

\caption{
IDK quality across benchmarks (7B).
IDK Precision $= \Pr[\text{truly beyond} \mid \text{says IDK}]$;
IDK Recall $= \Pr[\text{says IDK} \mid \text{truly beyond}]$;
IDK Rate $=$ fraction of questions where the model abstains.
}

\label{tab:idk_quality}

\end{table}
We compare Reflection-only with REIN to isolate the effect of boundary-aware abstention.
Both methods use the same reflection format, but REIN additionally uses the group-level boundary indicator and IDK reward.
Across the evaluated backbones, REIN further improves selective accuracy and reduces the hallucination proxy compared with Reflection-only.
This comes with a controlled reduction in coverage, because REIN learns to abstain on low-support prompts.
Effective accuracy remains close to Reflection-only, indicating that the reliability gain is not obtained by excessive refusal.

\subsection{Mechanism Analysis}

\paragraph{Two-level reliability-aware finalization.}
The ablation results suggest that REIN makes reliability judgments at two different levels. At the completion level, reflection assesses the solution implied by the preceding reasoning. At the prompt level, the boundary objective estimates whether the current policy can produce a verified solution under a fixed sampling budget. The distinction matters because a failed rollout does not necessarily mean that the prompt is beyond the model's capability: another rollout may still succeed. In such cases, correction is preferable to refusal. By contrast, when none of the sampled rollouts succeeds, repeated refinement may only make an unsupported answer more convincing. The two objectives therefore discourage both premature refusal and unjustified continued answering.

The current evidence supports completion-level reliability alignment and answer-level correction, but not step-level error localization. REIN receives no supervision indicating which token or reasoning step caused the failure. An incorrect draft followed by a correct final answer therefore shows that reflection can support a better final decision, but does not establish that the model identified the exact source of the error. Step-level diagnostics are reported separately in Appendix~A.

\paragraph{Selective boundary behavior.}
The difficulty-stratified results in Appendix~E show that abstention becomes more frequent as the tasks become harder. The IDK rate increases from 2.1\% on two-step GSM8K problems to 14.8\% on problems requiring at least five steps, and from 3.2\% on MATH Level~1 to 31.8\% on Level~5. ARC-Challenge shows a similar pattern across grade levels. Although benchmark difficulty is only an approximate indicator of policy-relative answerability, this monotonic trend argues against a uniform refusal strategy. The IDK precision of 78.4--83.2\% and recall of 65.7--72.8\% in Table~\ref{tab:idk_quality} further indicate that refusals are concentrated on low-support prompts rather than applied indiscriminately.

\section{Conclusion}

We presented REIN, a framework that improves the reliability of large reasoning models by enabling them to assess whether a generated solution should be trusted or rejected before finalizing an answer. REIN addresses two complementary failure modes of hallucination. Reflection alignment helps the model identify unreliable reasoning outcomes, while boundary-aware abstention encourages the model to avoid unsupported answers when the current capability is insufficient. Together, these signals provide a unified reliability-oriented training approach without requiring step-level supervision or inference-time intervention. Experiments across multiple backbones and reasoning benchmarks show that REIN reduces confident but incorrect responses while maintaining strong accuracy and coverage. These results highlight that reliable reasoning requires not only generating better solutions, but also knowing when a solution should not be trusted. Future work will investigate more fine-grained reliability assessment and broader uncertainty modeling.

\bibliographystyle{unsrtnat}
\bibliography{references}

\appendix

\section{Metric Definitions}
\label{sec:metric_appendix}

\subsection{Outcome Parsing and Notation}

For an evaluation set containing $N\geq 1$ instances, each method first generates one response using its native inference procedure. A deterministic output adapter extracts the final-answer span and maps it to exactly one of three mutually exclusive outcomes: a substantive answer, an explicit
\texttt{IDK}, or an invalid output. We define
\begin{equation}
\begin{aligned}
A_i &= \mathbb{1}[y_i\text{ is a substantive answer}],\\
D_i &= \mathbb{1}[y_i\text{ is an explicit \texttt{IDK}}],\\
U_i &= \mathbb{1}[y_i\text{ is invalid}],
\end{aligned}
\qquad
A_i+D_i+U_i=1.
\end{equation}

The abstention parser examines only the extracted final-answer span; uncertainty expressed in the reasoning or reflection does not constitute an abstention. An empty or missing answer, malformed required delimiters,
multiple conflicting answers, a mixture of an abstention and a substantive answer, or an answer that cannot be parsed for the target task is considered invalid. Method-specific output adapters may be used to accommodate different native output formats, but they are fixed before evaluation and cannot access reference answers or verifier feedback.

Let
\begin{equation}
T_i
=
\mathbb{1}
\left[
A_i=1
\ \wedge\
y_i\text{ passes the task-specific verifier}
\right].
\end{equation}
Thus, $T_i\leq A_i$, and $T_i=0$ whenever $A_i=0$. GSM8K uses normalized numerical equivalence, MATH-500 uses mathematical-equivalence checking, StrategyQA uses normalized exact matching over \texttt{yes}/\texttt{no}, and ARC-Challenge uses normalized exact matching of the selected option label.

We define
\begin{equation}
N_a=\sum_{i=1}^{N}A_i,
\qquad
N_c=\sum_{i=1}^{N}T_i,
\qquad
N_w=\sum_{i=1}^{N}(A_i-T_i),
\end{equation}
where $N_a$, $N_c$, and $N_w$ denote the numbers of substantive, correct
substantive, and incorrect substantive answers, respectively. Consequently,
$N_a=N_c+N_w$.

\subsection{Native Reflection Indicator}
For H-Proxy evaluation, every method generates a  structured\texttt{<think>} \(\rightarrow\) \texttt{<reflection>} \(\rightarrow\) \texttt{<answer>} completion in a single autoregressive pass. We use \emph{native} to mean that the reflection is emitted before the final answer within this same completion, regardless of whether the method was explicitly trained with a reflection objective. No additional post-hoc judging call is used to compute the primary H-Proxy. Each native completion is parsed into a pre-reflection draft $y_i^{(0)}$, a reflection $r_i$, and a final answer $y_i^{(1)}$. The draft $y_i^{(0)}$ is extracted only from the content preceding the reflection.

Let
\begin{equation}
P_i^{(0)}
=
\mathbb{1}
\left[
y_i^{(0)}
\text{ is a parseable substantive draft}
\right]
\end{equation}
denote whether a valid draft answer can be extracted. Draft extraction
failures are excluded from draft-correctness metrics and are reported
separately through the draft extraction rate.

Draft correctness is defined as
\begin{equation}
T_i^{(0)}
=
\mathbb{1}
\left[
P_i^{(0)}=1
\ \wedge\
y_i^{(0)}
\text{ passes the task-specific verifier}
\right].
\end{equation}

The native reflection stance is parsed as
\begin{equation}
s_i
\in
\{
\texttt{correct},
\texttt{wrong},
\texttt{uncertain},
\bot
\},
\end{equation}
where $\bot$ denotes an unparsable reflection. We define the binary
native reliability indicator as
\begin{equation}
J_i^{\mathrm{nat}}
=
\mathbb{1}
\left[
s_i=\texttt{correct}
\right].
\end{equation}
Thus, \texttt{wrong}, \texttt{uncertain}, and unparsable reflections
are not treated as endorsements. Uncertain and unparsable cases remain
separately visible in the native-reflection diagnostics.

Let
\begin{equation}
N_w^{(0)}
=
\sum_{i=1}^{N}
P_i^{(0)}(1-T_i^{(0)})
\end{equation}
denote the number of incorrect parseable drafts, and let
\begin{equation}
N_e^{\mathrm{nat}}
=
\sum_{i=1}^{N}
P_i^{(0)}(1-T_i^{(0)})J_i^{\mathrm{nat}}
\end{equation}
denote the number of incorrect drafts that are endorsed as
\texttt{correct} by their native reflections.

\subsection{Primary Evaluation Metrics}
\label{sec:primary_metric_definitions}

Since REIN allows the model to abstain, standard accuracy alone is not sufficient for evaluation.
A model can obtain higher accuracy on attempted questions simply by refusing many difficult questions.
Therefore, we evaluate reliability under a selective answering setting using four metrics:
Selective Accuracy, Hallucination Proxy (H-Proxy), Coverage, and Effective Accuracy.

\begin{equation*}
\begin{aligned}
\mathrm{Sel.Acc.}
&=
\frac{N_c}{N_a},
\\
\mathrm{H\text{-}Proxy}
&=
\frac{N_e^{\mathrm{nat}}}{N_w^{(0)}},
\\
\mathrm{Cov.}
&=
\frac{N_a}{N},
\\
\mathrm{Eff.Acc.}
&=
\mathrm{Sel.Acc.}\times\mathrm{Cov.}
=
\frac{N_c}{N}.
\end{aligned}
\end{equation*}

\textbf{Selective Accuracy} measures the correctness of answers that the model actually attempts.
It reflects whether the model gives more accurate answers when it chooses to answer.

\textbf{Hallucination Proxy} measures the fraction of incorrect pre-reflection drafts that are nevertheless endorsed as
\texttt{correct} by the native reflection generated within the same completion. It captures the failure mode in which the model incorrectly endorses an unreliable draft before producing its final answer.

\textbf{Coverage} measures the fraction of questions for which the model provides a substantive answer instead of abstaining or producing an invalid output.

\textbf{Effective Accuracy} measures the overall fraction of questions answered correctly after accounting for abstention.
It combines answer correctness and coverage, and therefore reflects the actual utility of the model under selective answering.

Higher Selective Accuracy, Coverage, and Effective Accuracy are better, while lower Hallucination Proxy is better.
Together, these metrics show whether REIN improves reliability without relying on excessive abstention.

\subsection{Edge Cases and Reporting}

If $N_a=0$, Coverage and Effective Accuracy are defined as zero,
whereas Selective Accuracy is undefined because its denominator is
zero. Because H-Proxy is computed from pre-reflection drafts, its
definition does not depend on $N_a$.

If $N_w^{(0)}=0$, H-Proxy is undefined and is reported as ``--''.
Other conditional reliability metrics are likewise reported as
``--'' whenever their corresponding denominators are zero.

All metrics are computed directly from sample-level events and are
reported as percentages. Algebraic identities are used only for
interpretation; final values are not reconstructed by multiplying
rounded table entries. Metrics are computed separately for each
checkpoint--benchmark pair, and every cross-benchmark aggregate
explicitly states whether macro- or micro-averaging is used.

\section{Supplementary Experimental Results}

\begin{table}[!t]
\centering
\small
\renewcommand{\arraystretch}{1.1}
\setlength{\tabcolsep}{0.8pt}

\resizebox{\textwidth}{!}{%
\begin{tabular}{l *{4}{ccc}}
\toprule

\multirow{3}{*}{\textbf{Method}}
& \multicolumn{6}{c}{\textbf{Mathematical Reasoning}}
& \multicolumn{6}{c}{\textbf{Commonsense Reasoning}} \\

\cmidrule(lr){2-7}
\cmidrule(lr){8-13}

& \multicolumn{3}{c}{\textbf{GSM8K}}
& \multicolumn{3}{c}{\textbf{MATH}}
& \multicolumn{3}{c}{\textbf{StrategyQA}}
& \multicolumn{3}{c}{\textbf{ARC-Challenge}} \\

\cmidrule(lr){2-4}
\cmidrule(lr){5-7}
\cmidrule(lr){8-10}
\cmidrule(lr){11-13}

& \textbf{Sel.Acc.}$\uparrow$
& \textbf{H-Proxy}$\downarrow$
& \textbf{Eff.Acc.}$\uparrow$

& \textbf{Sel.Acc.}$\uparrow$
& \textbf{H-Proxy}$\downarrow$
& \textbf{Eff.Acc.}$\uparrow$

& \textbf{Sel.Acc.}$\uparrow$
& \textbf{H-Proxy}$\downarrow$
& \textbf{Eff.Acc.}$\uparrow$

& \textbf{Sel.Acc.}$\uparrow$
& \textbf{H-Proxy}$\downarrow$
& \textbf{Eff.Acc.}$\uparrow$ \\

\midrule

% ============================================================
% Mistral-7B-v0.3
% ============================================================

\rowcolor{graybg}
\multicolumn{13}{l}{\textbf{Backbone: Mistral-7B-v0.3}} \\

Base
& 51.4 & 51.0 & 51.4
& 12.8 & 55.0 & 12.8
& 58.5 & 43.1 & 58.5
& 65.4 & 42.2 & 65.4 \\

IDK Prompting
& 54.8 & 34.0 & 45.2$_{(82.5)}$
& 15.8 & 43.1 & 9.2$_{(58.4)}$
& 62.4 & 32.3 & 49.3$_{(79.0)}$
& 69.5 & 30.1 & 59.0$_{(84.9)}$ \\

Reflection-only
& 54.5 & 32.1 & \textbf{54.5}
& 15.2 & 36.1 & \textbf{15.2}
& 62.0 & 28.7 & \textbf{62.0}
& 69.0 & 27.1 & \textbf{69.0} \\

\textbf{REIN (Ours)}
& \textbf{58.5} & \textbf{17.9} & 52.7$_{(90.1)}$
& \textbf{18.5} & \textbf{25.1} & 13.4$_{(72.4)}$
& \textbf{65.7} & \textbf{19.7} & 59.4$_{(90.4)}$
& \textbf{72.2} & \textbf{18.1} & 66.0$_{(91.4)}$ \\

\midrule

% ============================================================
% DeepSeek-R1-Distill-8B
% ============================================================

\rowcolor{graybg}
\multicolumn{13}{l}{\textbf{Backbone: DeepSeek-R1-Distill-8B}} \\

Base
& 76.8 & 44.0 & 76.8
& 38.5 & 48.0 & 38.5
& 66.7 & 37.0 & 66.7
& 76.2 & 35.0 & 76.2 \\

IDK Prompting
& 89.0 & 19.0 & 81.4$_{(91.5)}$
& 76.5 & 17.0 & 62.7$_{(82.0)}$
& 73.8 & 23.0 & 61.6$_{(83.5)}$
& 83.2 & 24.0 & 71.1$_{(85.5)}$ \\

Reflection-only
& 87.8 & 15.0 & \textbf{87.8}
& 74.3 & 13.0 & \textbf{74.3}
& 70.8 & 20.0 & \textbf{70.8}
& 81.2 & 18.0 & 81.2 \\

\textbf{REIN (Ours)}
& \textbf{91.8} & \textbf{8.0} & 87.7$_{(95.5)}$
& \textbf{80.5} & \textbf{9.0} & 74.1$_{(92.0)}$
& \textbf{76.8} & \textbf{14.0} & 70.7$_{(92.0)}$
& \textbf{86.0} & \textbf{10.0} & \textbf{81.3}$_{(94.5)}$ \\

\bottomrule
\end{tabular}%
}

\label{tab:additional_results}

\caption{
Supplementary experimental results on Mistral-7B-v0.3 and DeepSeek-R1-Distill-8B.
}
\end{table}

\subsection{Comparison with Existing Methods}

We compare REIN with existing methods for mitigating reasoning hallucinations in figure \ref{fig:existing_method_comparison}, including In-Context Learning (ICL), knowledge distillation, and BARREL. To ensure a fair comparison, we adopt three same key evaluation metrics with BARREL: Accuracy (Acc.), Truthfulness (Truth.), and Reliability (Rel.), which are defined as follows: 
$\text{Acc.} = \frac{N_c}{N}$, 
$\text{Truth.} = \frac{N_c + N_r}{N}$, and 
$\text{Rel.} = \text{ans.} \cdot \text{Truth.} + (1 - \text{ans.}) \cdot \text{Acc.}$, 
where $\text{ans.} = 1 - \frac{N_r}{N}$. 
As illustrated in Figure~\ref{fig:existing_method_comparison}, REIN achieves the highest Acc., Truth., and Rel. among all compared baselines. Furthermore, we also compare our method against token-level calibration approaches such as R-Tuning and IDK-token training. We observe that these methods significantly interfere with the model's original output distribution, thereby degrading its inherent reasoning capabilities.

\begin{figure}[!t]
    \centering
    \includegraphics[width=0.75\textwidth]{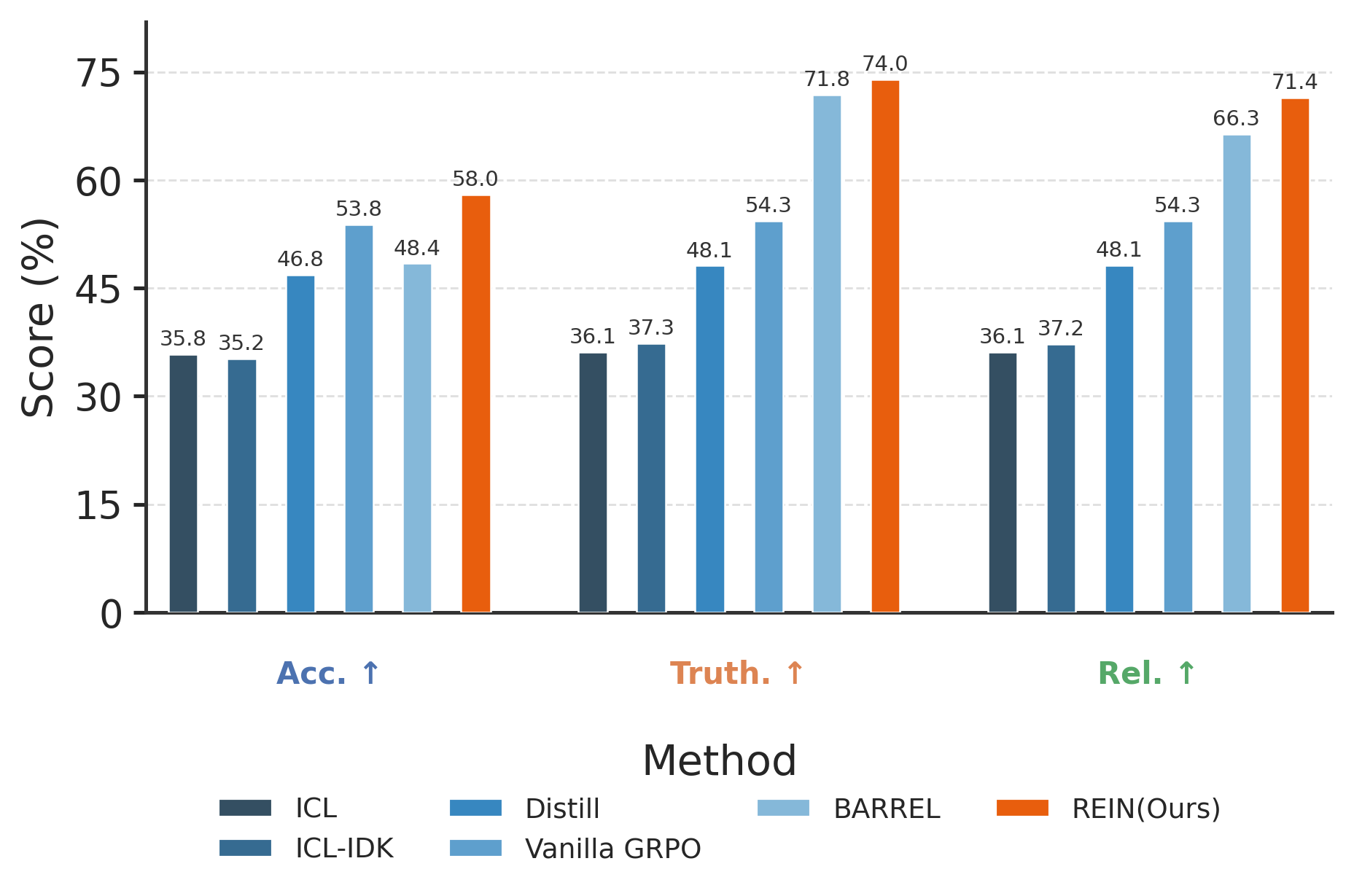}
    \caption{Comparison of REIN with existing methods. REIN achieves the highest Accuracy, Truthfulness, and Reliability. The model is Deepseek-Distill-8B and the benchmark is TriviaQA.}
    \label{fig:existing_method_comparison}
\end{figure}

\subsection{Extended Baseline Comparison under a Unified Protocol}
\label{sec:extended_baseline_comparison}

\begin{table}[!t]
\centering
\small
\renewcommand{\arraystretch}{1.05}
\setlength{\tabcolsep}{0.5pt}

\resizebox{\textwidth}{!}{%
\begin{tabular}{l *{4}{ccc}}
\toprule

\multirow{3}{*}{\textbf{Method}}
& \multicolumn{6}{c}{\textbf{Mathematical Reasoning}}
& \multicolumn{6}{c}{\textbf{Commonsense Reasoning}} \\

\cmidrule(lr){2-7}
\cmidrule(lr){8-13}

& \multicolumn{3}{c}{\textbf{GSM8K}}
& \multicolumn{3}{c}{\textbf{MATH}}
& \multicolumn{3}{c}{\textbf{StrategyQA}}
& \multicolumn{3}{c}{\textbf{ARC-Challenge}} \\

\cmidrule(lr){2-4}
\cmidrule(lr){5-7}
\cmidrule(lr){8-10}
\cmidrule(lr){11-13}

& \textbf{Sel.Acc.}$\uparrow$
& \textbf{H-Proxy}$\downarrow$
& \textbf{Eff.Acc.}$\uparrow$

& \textbf{Sel.Acc.}$\uparrow$
& \textbf{H-Proxy}$\downarrow$
& \textbf{Eff.Acc.}$\uparrow$

& \textbf{Sel.Acc.}$\uparrow$
& \textbf{H-Proxy}$\downarrow$
& \textbf{Eff.Acc.}$\uparrow$

& \textbf{Sel.Acc.}$\uparrow$
& \textbf{H-Proxy}$\downarrow$
& \textbf{Eff.Acc.}$\uparrow$
\\

\midrule

% ============================================================
% Qwen2.5-7B
% ============================================================

\rowcolor{graybg}
\multicolumn{13}{l}{\textbf{Backbone: Qwen2.5-7B}} \\

\textit{Training-free} \\

Base
& 79.1 & 41.0 & 79.1
& 41.2 & 45.0 & 41.2
& 64.3 & 38.0 & 64.3
& 74.8 & 36.0 & 74.8 \\

Self-Refine
& 78.4 & 39.0 & 78.4
& 40.6 & 44.0 & 40.6
& 65.1 & 36.0 & 65.1
& 74.9 & 35.0 & 74.9 \\

Self-Consist@16
& 87.3 & 43.0 & 87.3
& 50.4 & 47.0 & 50.4
& 69.8 & 40.0 & 69.8
& 79.1 & 38.0 & 79.1 \\

IDK Prompting
& 86.8 & 22.5 & 78.9$_{(90.9)}$
& 54.1 & 31.1 & 38.7$_{(71.5)}$
& 73.4 & 22.0 & 60.2$_{(82.0)}$
& 84.2 & 20.8 & 76.7$_{(91.1)}$
\\

\textit{Fine-tuning} \\

RFT
& 84.7 & 34.0 & 84.7
& 47.9 & 39.0 & 47.9
& 68.9 & 32.0 & 68.9
& 78.6 & 30.0 & 78.6 \\

DPO
& 82.9 & 36.0 & 82.9
& 46.1 & 41.0 & 46.1
& 67.5 & 33.0 & 67.5
& 77.4 & 31.0 & 77.4 \\

R-Tuning
& 84.2 & 13.5 & 68.8$_{(81.7)}$
& 51.2 & 27.5 & 40.7$_{(79.4)}$
& 64.3 & 12.4 & 57.5$_{(89.4)}$
& 86.1 & 17.2 & 82.0$_{(95.2)}$
\\

\textit{RL-based} \\

Vanilla GRPO
& 86.4 & 35.0 & 86.4
& 49.7 & 40.0 & 49.7
& 68.4 & 33.0 & 68.4
& 78.2 & 31.0 & 78.2 \\

$S^2$R
& 88.5 & 31.0 & 88.5
& 53.8 & 37.0 & \textbf{53.8}
& 74.2 & 28.0 & \textbf{74.2}
& 79.8 & 27.0 & 79.8 \\

BARREL
& 87.4 & 21.0 & 79.7$_{(91.2)}$
& 56.2 & 27.0 & 44.7$_{(79.6)}$
& 73.5 & 20.0 & 64.8$_{(88.1)}$
& 84.1 & 19.0 & 76.0$_{(90.4)}$
\\

TruthRL
& 86.9 & 17.9 & 80.6$_{(92.7)}$
& 57.3 & 42.7 & 48.7$_{(85.0)}$
& 67.7 & 26.9 & 62.7$_{(92.6)}$
& 83.7 & 10.8 & 77.5$_{(92.6)}$
\\

\textit{Ours} \\

Reflection-only
& 85.2 & 18.9 & \textbf{85.2}
& 51.2 & 22.1 & 51.2
& 70.3 & 17.5 & \textbf{70.3}
& 82.1 & 16.2 & \textbf{82.1}
\\

\textbf{REIN}
& \textbf{90.6} & \textbf{11.0} & 85.0$_{(93.8)}$
& \textbf{61.0} & \textbf{15.0} & 50.9$_{(83.5)}$
& \textbf{76.2} & \textbf{10.0} & 69.2$_{(90.8)}$
& \textbf{86.5} & \textbf{9.0} & 80.9$_{(93.5)}$
\\

% ============================================================
% LLaMA-3.1-8B
% ============================================================

\midrule

\rowcolor{graybg}
\multicolumn{13}{l}{\textbf{Backbone: LLaMA-3.1-8B}} \\

\textit{Training-free} \\

Base
& 76.8 & 44.0 & 76.8
& 38.5 & 48.0 & 38.5
& 66.7 & 37.0 & 66.7
& 76.2 & 35.0 & 76.2 \\

Self-Refine
& 76.1 & 42.0 & 76.1
& 37.8 & 47.0 & 37.8
& 67.4 & 35.0 & 67.4
& 76.8 & 34.0 & 76.8 \\

Self-Consist@16
& 85.6 & 46.0 & 85.6
& 47.9 & 50.0 & 47.9
& 71.5 & 39.0 & 71.5
& 80.3 & 37.0 & 80.3 \\

IDK Prompting
& 85.0 & 25.2 & 74.2$_{(87.3)}$
& 51.8 & 33.9 & 36.8$_{(71.0)}$
& 74.8 & 21.1 & 67.4$_{(90.1)}$
& 86.2 & 20.5 & 76.2$_{(88.4)}$ \\

\textit{Fine-tuning} \\

RFT
& 83.2 & 36.0 & 83.2
& 45.1 & 42.0 & 45.1
& 70.6 & 31.0 & 70.6
& 79.7 & 29.0 & 79.7 \\

DPO
& 81.4 & 38.0 & 81.4
& 43.6 & 44.0 & 43.6
& 69.2 & 33.0 & 69.2
& 78.5 & 31.0 & 78.5 \\

R-Tuning
& 86.5 & 21.4 & 77.9$_{(90.0)}$
& 52.9 & 28.9 & 44.8$_{(84.6)}$
& 73.8 & 23.1 & 69.0$_{(93.5)}$
& 79.6 & 18.4 & 77.8$_{(97.7)}$ \\

\textit{RL-based} \\

Vanilla GRPO
& 84.9 & 37.0 & 84.9
& 47.2 & 43.0 & 47.2
& 70.1 & 32.0 & 70.1
& 79.4 & 30.0 & 79.4 \\

$S^2$R
& 86.8 & 33.0 & 86.8
& 53.4 & 39.0 & \textbf{53.4}
& 75.6 & 30.0 & \textbf{75.6}
& 81.0 & 28.0 & 81.0 \\

BARREL
& 86.1 & 22.0 & 77.9$_{(90.5)}$
& 54.7 & 28.0 & 42.8$_{(78.2)}$
& 75.1 & 21.0 & 65.6$_{(87.4)}$
& 85.8 & 20.0 & 78.2$_{(91.1)}$ \\

TruthRL
& 77.8 & 22.1 & 77.3$_{(99.3)}$
& 53.2 & 50.6 & 47.3$_{(89.0)}$
& 73.7 & 19.2 & 68.5$_{(92.9)}$
& 82.5 & 17.2 & 81.3$_{(98.6)}$ \\

\textit{Ours} \\

Reflection-only
& 83.5 & 21.2 & 83.5
& 49.2 & 26.0 & 49.2
& 72.9 & 17.7 & 72.9
& 86.2 & 15.9 & \textbf{86.2} \\

\textbf{REIN}
& \textbf{89.2} & \textbf{12.0} & 83.0$_{(93.0)}$
& \textbf{59.5} & \textbf{18.0} & 48.5$_{(81.5)}$
& \textbf{78.2} & \textbf{11.0} & 72.3$_{(92.5)}$
& \textbf{88.0} & \textbf{10.0} & 83.6$_{(95.0)}$ \\

\bottomrule
\end{tabular}%
}

\caption{Comprehensive experimental results on Qwen2.5-7B and LLaMA-3.1-8B.}
\label{tab:extend_results}
\end{table}

Table \ref{tab:extend_results} places REIN alongside nine published methods spanning
training-free self-correction, refusal-aware fine-tuning, and
RL-based reliability alignment, together with the Base checkpoint
and the IDK-prompting control of Table~1 of the main paper.
The results separate two quantities that the literature usually
reports together. Reasoning-oriented optimization reliably buys
accuracy: vanilla GRPO lifts average effective accuracy from
64.9 to 70.7 on Qwen2.5-7B, and $S^2R$, which explicitly
rewards self-verification and self-correction, reaches the highest
effective accuracy of any method in the table at 74.1. Neither
buys reliability: their H-Proxy stays at 34.8 and 30.8, barely
below the 40.0 of the untrained backbone, so roughly one wrong
answer in three is still delivered with an explicit assertion that
it is sound. Test-time sampling is not an alternative route to it
either---sixteen-sample self-consistency adds 6.8 points of
effective accuracy and simultaneously raises H-Proxy to 42.0,
because majority voting removes exactly the errors the model was
least sure of and leaves behind the self-consistent ones that
reflection is most willing to endorse. Abstention supervision
moves the second quantity but at a cost on the first: R-Tuning
reaches the lowest H-Proxy of any external method (17.7) yet
abstains hard enough that its effective accuracy of 62.3 falls
below the untrained baseline, and BARREL, whose IDK recipe was
calibrated on long-CoT models that guess at the last moment,
transfers only partially to short-CoT instruction-tuned backbones
and settles at a conservative operating point of 87.3 coverage
and 66.3 effective accuracy.

The most informative comparison is with TruthRL, the strongest
RL-based abstention baseline, and here the two quantities come
apart cleanly. On Qwen2.5-7B the two methods sit at almost the
same coverage (90.7 against our 90.4), and at that shared
operating point REIN adds 4.7 points of selective accuracy and
4.1 points of effective accuracy while cutting H-Proxy from
24.6 to 11.3. On LLaMA-3.1-8B TruthRL retains 4.5 more
coverage points (95.0 against 90.5) and still converts them into
3.3 fewer correct answers over the full evaluation set
(68.6 against 71.9). The advantage is not carried by one
benchmark: REIN attains higher effective accuracy in all eight
backbone--benchmark cells, and the largest single margin,
6.5 points, appears on StrategyQA with Qwen2.5-7B, where
TruthRL answers 92.6\% of the questions and endorses 26.9\%
of the wrong answers it delivers.

That contrast also localizes the mechanism. TruthRL's ternary
reward leaves H-Proxy above our reflection-only ablation on both
backbones (24.6 against 18.7 on Qwen2.5-7B, 27.3 against
20.2 on LLaMA-3.1-8B), which is what one expects of any signal
that prices whether to answer without constraining how the model
assesses the answer it does give; only the full objective, in which
the reflection-veracity term aligns the stance of an explicit
reflection span with verified correctness while the boundary-aware
term redirects the answer distribution itself, brings the rate below
20.0 on every benchmark and every backbone, and to 12.0 or
below outside MATH. What REIN does not claim is uniform
superiority on utility alone: $S^2R$ answers every question and
still reaches 74.1 average effective accuracy against our 71.5,
so a deployment that never pays for a wrong answer is better
served by refusing nothing. The case for REIN is that at a
comparable answer rate it is right more often and, when it is
wrong, says so.

\subsection{Supplementary MCC Results}

\paragraph{Supplementary reflection-aware metric}
We also present the Matthews Correlation Coefficient (MCC) as a secondary measure of reflection-aware reliability because it measures how much agreement there is between whether a model's output is correct or the output is reliably classified according to its reflection judgement (reliable vs.
{
\begin{equation*}
\mathrm{MCC}
=
\frac{\mathrm{TP}\cdot \mathrm{TN} - \mathrm{FP}\cdot \mathrm{FN}}
{\sqrt{(\mathrm{TP}+\mathrm{FP})(\mathrm{TP}+\mathrm{FN})(\mathrm{TN}+\mathrm{FP})(\mathrm{TN}+\mathrm{FN})}}
\end{equation*}
} unreliable). Here solution correctness will be used as the ground-truth label and the reflection judgment will be used as a binary prediction. In terms of definitions: TP, TN, FP, and FN are defined as true positives, true negatives, false positives, and false negatives respectively. Therefore, the MCC can be calculated as
In our scenario, a positive label means that a solution is correct and the associated reflection judgement labels the solution as reliable, whereas a negative label indicates that the solution is incorrect and the associated reflection judgement labels the solution as unreliable. The MCC does not depend on whether the two classes of examples are equally distributed (i.e., has class imbalance) and therefore will give penalties to both false-positive (incorrect answers that are classified by the reflection judgement as being reliable) and false-negative (correct answers that are classified by the reflection judgement as being unreliable) classifications. Where indicated, we map correct to reliable and incorrect/uncertain to unreliable.

\subsection{Evolution of Reward Components}

\begin{figure}[!t]
    \centering
    \includegraphics[width=0.75\textwidth]{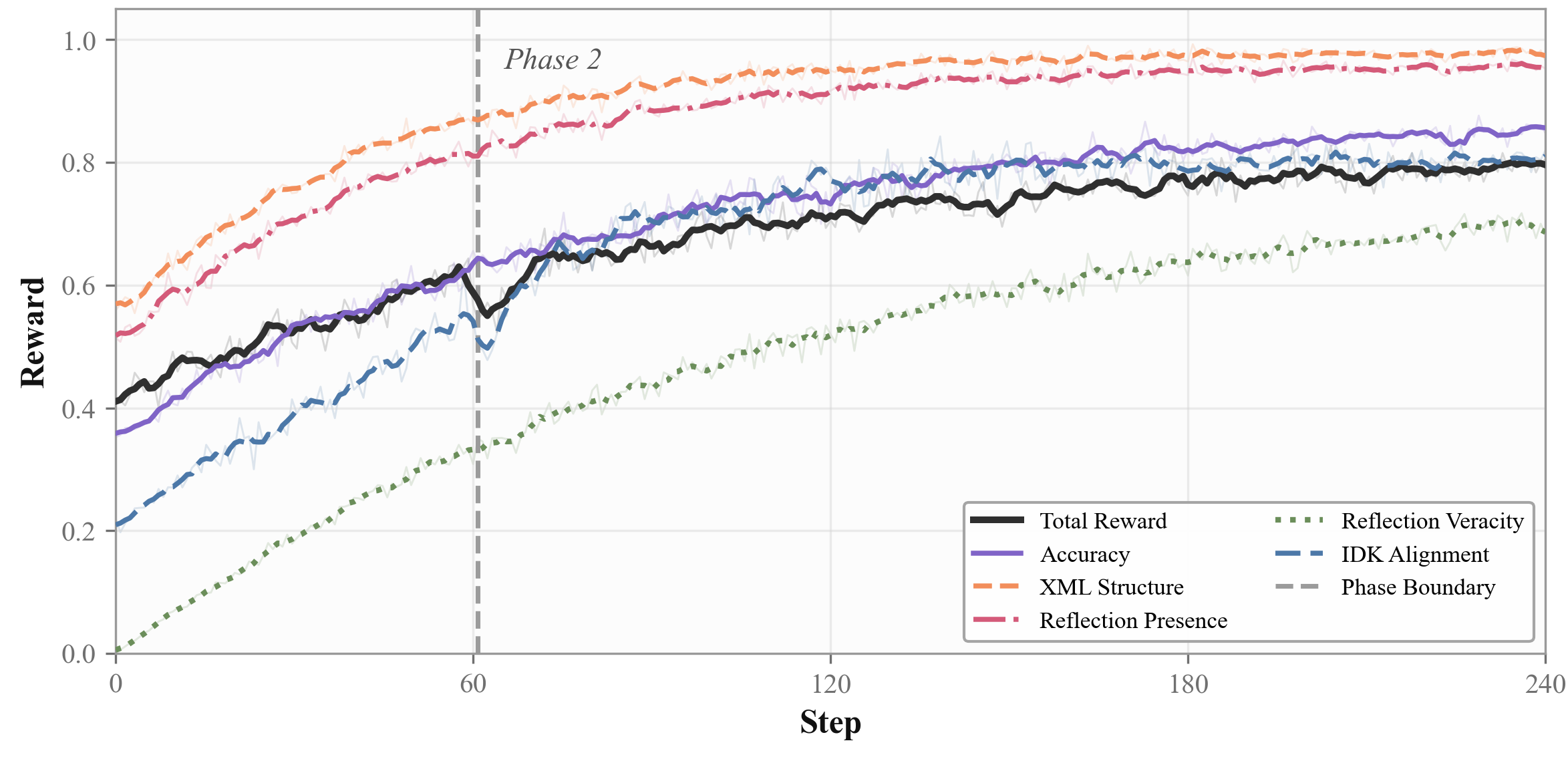}
    \caption{Reward Composition Dynamics During REIN Training (Qwen2.5-7B)}
    \label{fig:reward_components}
\end{figure}

Figure~\ref{fig:reward_components} illustrates the evolution of the
overall reward and its five components during REIN training.
The training exhibits two distinct phases. In the first phase,
the model rapidly acquires the required response format and
reasoning behavior, as reflected by the fast convergence of the
XML structure and reflection presence rewards. Meanwhile, the
accuracy reward steadily improves, indicating that the enforced
reflection mechanism does not compromise task-solving ability.

After the phase transition, the reflection veracity and IDK
alignment rewards continue to increase, suggesting that the model
gradually learns to calibrate its self-evaluation and align its
abstention behavior with its actual capability boundary. The
consistent improvement of total reward demonstrates that REIN
jointly optimizes answer correctness, reliable self-reflection,
and selective abstention rather than optimizing these objectives
independently.

\begin{table}[t]
\centering
\small
\setlength{\tabcolsep}{2pt}

\resizebox{0.75\textwidth}{!}{%
\begin{tabular}{lccccc}
\toprule
\multicolumn{6}{c}{
\textbf{(a) One model, one fixed answer set, seven abstention scores}
}\\
\midrule

\textbf{Selective score}
& \textbf{AUROC}$\uparrow$
& \textbf{AURC}$\downarrow$
& \textbf{Sel.@90}$\uparrow$
& \textbf{Sel.@80}$\uparrow$
& \textbf{\#fwd}
\\
\midrule

\multicolumn{6}{l}{\textit{GSM8K (raw=88.1, risk@100=11.9)}}\\

Sequence log-prob
& .712 & .063 & 89.6 & 91.2 & 1 \\

Length-norm. log-prob
& .741 & .058 & 90.1 & 92.0 & 1 \\

P(True)
& .779 & .051 & 90.8 & 93.1 & 2 \\

Semantic entropy @16
& .836 & .043 & 91.6 & 94.5 & 16 \\

Self-consistency agr. @16
& .849 & .040 & 91.9 & 94.9 & 16 \\

Reflection stance (ours)
& .857 & .038 & 92.2 & 95.3 & 1 \\

+ self-consistency @16
& \textbf{.884} & \textbf{.031}
& \textbf{93.1}
& \textbf{96.4}
& 17 \\

\midrule

\multicolumn{6}{l}{\textit{MATH (raw=55.0, risk@100=45.0)}}\\

Sequence log-prob
& .688 & .288 & 56.4 & 58.2 & 1 \\

Length-norm. log-prob
& .703 & .281 & 57.0 & 59.1 & 1 \\

P(True)
& .724 & .272 & 57.6 & 60.2 & 2 \\

Semantic entropy @16
& .781 & .259 & 58.4 & 61.8 & 16 \\

Self-consistency agr. @16
& .792 & .255 & 58.7 & 62.4 & 16 \\

Reflection stance (ours)
& .786 & .257 & 58.5 & 62.0 & 1 \\

+ self-consistency @16
& \textbf{.818} & \textbf{.246}
& \textbf{59.3}
& \textbf{63.5}
& 17 \\

\bottomrule
\end{tabular}%
}

\caption{Selective prediction with different abstention scores. The native reflection stance provides a competitive single-pass selective signal, outperforming conventional likelihood-based confidence measures and matching or exceeding sixteen-sample self-consistency on GSM8K. Its combination with self-consistency yields the strongest selective performance, indicating that reflection-based diagnosis and sample agreement provide complementary reliability information.}
\label{tab:selective_prediction_a}
\end{table}

\begin{table}[t]
\centering
\small

\resizebox{0.75\textwidth}{!}{%
\begin{tabular}{lcccc}
\toprule

\multicolumn{5}{c}{
\textbf{(b) Each method calibrated to Cov.=90 (GSM8K)}
}\\

\midrule

\textbf{Method}
& \textbf{Raw}
& \textbf{Sel.@90}$\uparrow$
& \textbf{$\Delta$ over raw}$\uparrow$
& \textbf{H-Proxy@90}$\downarrow$
\\

\midrule

R-Tuning
&79.4&82.0&+2.6&16.5\\

BARREL
&85.0&87.7&+2.7&20.0\\

TruthRL
&84.0&88.0&+4.0&17.0\\

\textbf{REIN}
&88.1&\textbf{92.2}&\textbf{+4.1}&\textbf{10.0}\\

\bottomrule

\end{tabular}%
}

\caption{Selective prediction comparison under a fixed coverage of 90\%.}
\label{tab:selective_prediction_b}
\end{table}

\section{Ablation and Diagnostic Analysis}
\label{app:additional_ablation}

\begin{table}[!t]
\centering
\small

\resizebox{0.75\textwidth}{!}{%
\begin{tabular}{@{}lcccc@{}}
\toprule
Configuration & Sel.Acc$\uparrow$ & H-Proxy $\downarrow$ & Cov.$\uparrow$ & Eff.Acc \\
\midrule
Full REIN            & \textbf{90.6} & \textbf{0.11} & 93.8 & \textbf{85.0} \\
w/o Accuracy        & 78.5 & 0.16 & 92.0 & 72.2 \\
w/o XML Structure   & 88.4 & 0.14 & 93.0 & 82.2 \\
w/o Refl.\ Presence & 89.8 & 0.13 & 93.4 & 83.9 \\
w/o Refl.\ Veracity & 90.2 & 0.20 & 93.5 & 84.3 \\
w/o IDK alignment   & 85.2 & 0.19 & 100  & 85.2 \\
\bottomrule
\end{tabular}%
}
\caption{
Ablation study of REIN's reliability alignment components on GSM8K (Qwen2.5-7B).
Each row removes one objective component.
Removing Reflection Veracity increases hallucination errors,
while removing IDK alignment affects answer-abstain calibration.
}
\label{tab:reward_ablation}
\end{table}

\paragraph{Key findings}
(1) \textbf{Accuracy reward is essential}: removing it causes a large drop in selective accuracy ($-12.1$), confirming that correctness signal is the primary learning driver.
(2) \textbf{Reflection Veracity controls hallucination}: removing it has minimal effect on selective accuracy ($-0.4$) but nearly doubles the hallucination rate (0.11 $\rightarrow$ 0.20), showing this reward is specifically responsible for aligning the reflection's self-assessment with actual correctness.
(3) \textbf{IDK alignment enables abstention}: without it, the model never abstains (Cov.=100\%) and achieves identical results to ``REIN w/o IDK'' in the main table, confirming this reward is the sole driver of IDK behavior.
(4) \textbf{XML Structure supports reflection quality}: removing it reduces effective accuracy by 2.8 points and increases hallucination (0.11 $\rightarrow$ 0.14), indicating that consistent formatting is a prerequisite for effective self-checking.

\subsection{Sensitivity to Group Size $K$}
\label{app:k_ablation}

The knowledge boundary indicator $B(x)$ depends on the group size $K$: larger $K$ provides a more reliable estimate of whether a question is truly beyond the model's capability.
Table~\ref{tab:k_ablation} shows the effect of varying $K$ on GSM8K and MATH.

Larger $K$ yields more accurate boundary detection on both benchmarks ($+11.3$ on GSM8K, $+13.2$ on MATH from $K$=4 to 16), which translates to better selective accuracy and lower hallucination rate. Boundary detection is inherently harder on MATH (85.3\% vs.\ 89.7\% at $K$=16) due to its broader difficulty distribution, where more questions lie near the model's capability frontier. We use $K=16$ as the default, which achieves a good balance between detection accuracy and computational efficiency.

\begin{table}[!htbp]
\centering
\small

\resizebox{0.5\textwidth}{!}{%
\begin{tabular}{@{}clcccc@{}}
\toprule
Dataset & $K$ & Sel.Acc$\uparrow$ & Halluc.$\downarrow$ & Cov.$\uparrow$ & Bnd.Acc \\
\midrule
\multirow{3}{*}{GSM8K}
 & 4  & 87.3 & 0.15 & 95.8 & 78.4 \\
 & 8  & 89.1 & 0.13 & 94.5 & 84.2 \\
 & 16 & \textbf{90.6} & \textbf{0.11} & 93.8 & \textbf{89.7} \\
\midrule
\multirow{3}{*}{MATH}
 & 4  & 56.5 & 0.21 & 85.4 & 72.1 \\
 & 8  & 59.1 & 0.18 & 84.0 & 79.5 \\
 & 16 & \textbf{61.4} & \textbf{0.16} & 82.7 & \textbf{85.3} \\
\bottomrule
\end{tabular}%
}
\caption{Effect of group size $K$ on boundary estimation and downstream metrics. Larger $K$ improves boundary detection but increases sampling cost.}
\label{tab:k_ablation}
\end{table}

\subsection{Difficulty-Stratified Abstention}
\label{app:idk_difficulty}

\begin{table}[H]
\centering
\small

\resizebox{0.5\textwidth}{!}{%
\begin{tabular}{@{}llc@{}}
\toprule
Dataset & Difficulty Bin & IDK Rate (\%) \\
\midrule

\multirow{3}{*}{GSM8K (by steps)}
 & 2-step    & 2.1 \\
 & 3--4 step & 5.4 \\
 & 5+ step   & 14.8 \\

\midrule

\multirow{5}{*}{MATH (by level)}
 & Level 1 & 3.2 \\
 & Level 2 & 8.5 \\
 & Level 3 & 15.1 \\
 & Level 4 & 22.4 \\
 & Level 5 & 31.8 \\

\midrule

\multirow{3}{*}{ARC-C (by grade)}
 & Elementary    & 2.8 \\
 & Middle school & 6.5 \\
 & High school   & 13.4 \\

\bottomrule
\end{tabular}%
}

\caption{
IDK rate stratified by difficulty across benchmarks (7B).
The model's abstention rate increases monotonically with question difficulty,
confirming calibrated knowledge-boundary awareness.
}

\label{tab:idk_difficulty}

\end{table}

\section{Reflection Causality and Draft-to-Final Transition}
\label{sec:reflection_causality}

Section~4 of the main paper attributes REIN's reliability improvement to
the \texttt{<reflection>} span. This section investigates whether
reflection causally affects the downstream finalization process.
Rather than treating reflection as an auxiliary textual annotation,
we perform inference-time interventions on the reflection span and
trace how different reflection states influence the transition from
the initial draft $y_0$ to the final delivered answer $y_1$.

\subsection{Draft-to-Final Transition Analysis}
\label{sec:completion_transition}

To further analyze how reflection changes the final output, we trace
each instance from the initial draft answer $y_0$ to the final delivered
answer $y_1$.

\begin{table}[t]
\centering
\small
\renewcommand{\arraystretch}{1.1}
\setlength{\tabcolsep}{5pt}

\resizebox{0.5\textwidth}{!}{%
\begin{tabular}{lrrrr}
\toprule
\textbf{Draft $y_0$}
& $y_1$ correct
& $y_1$ wrong
& $y_1$ IDK
& \textbf{Total}
\\
\midrule

correct
& 2399
& 22
& 33
& 2454
\\

wrong
& 82
& 455
& 210
& 747
\\

unparsable
& 0
& 0
& 19
& 19
\\

\midrule

\textbf{Total}
& 2481
& 477
& 262
& 3220
\\

\bottomrule
\end{tabular}%
}

\caption{
Draft-to-final transition matrix of REIN on Qwen2.5-7B pooled over
the four benchmarks. Rows indicate the correctness status of the
initial draft $y_0$, while columns indicate the final delivered outcome
$y_1$. The matrix characterizes how reflection-guided finalization
transforms initial reasoning states into final decisions.
}
\label{tab:draft_final_transition}
\end{table}

Among 747 initially incorrect drafts, 82 become correct after
reflection, yielding a repair rate of

\[
\frac{82}{747}=11.0\%.
\]

In contrast, only 22 out of 2454 initially correct drafts become
incorrect, corresponding to a corruption rate of

\[
\frac{22}{2454}=0.9\%.
\]

Therefore, reflection improves finalization quality by correcting
incorrect reasoning outcomes substantially more often than it damages
correct ones.

The transition analysis also shows that abstention is targeted rather
than uniform. Among the 262 non-substantive final outputs, most originate
from incorrect or unreliable initial drafts. This indicates that the
abstention mechanism operates jointly with reflection: unreliable
reasoning trajectories are more likely to be removed rather than forced
into unsupported final answers.

\section{Finite-Sample Interpretation and Rollout-Budget Sensitivity}
\label{app:finite_sample}
\paragraph{Notation.}
We use the same notation as in Section 3.
For a prompt $x$, $y^\star(x)$ denotes the gold answer,
$T(x,y)=\mathbb{1}[y=y^\star(x)]$ denotes answer correctness,
and $p(x)=\mathbb{E}[T(x,y)\mid x]$ denotes the per-prompt correctness probability.
The parsed reflection judgment is denoted by $J(r)\in\{0,1\}$, where $1$ means reliable and $0$ means unreliable.

\subsection{Finite-Sample Interpretation and Rollout-Budget Sensitivity}
\label{app:k_sensitivity}

Let
\begin{equation}
p_\theta(x)
=
\Pr_{z\sim\pi_\theta(\cdot\mid x)}
\left[T(x,y_1)=1\right]
\end{equation}
denote the probability that the current policy produces a verified-correct final answer for prompt $x$. Under conditionally independent sampling, the probability that none of the $K$ completions succeeds is
\begin{equation}
\Pr\left(
\widehat{B}_K(x)=\texttt{beyond}
\mid x
\right)
=
\left(1-p_\theta(x)\right)^K.
\end{equation}

Therefore, a no-success group does not imply that $p_\theta(x)=0$. Instead, it provides increasingly strong evidence of a low policy-level success probability as the rollout budget increases.

When zero successes are observed among $K$ samples, a one-sided $(1-\alpha)$ upper confidence bound on the success probability is
\begin{equation}
p_\theta(x)
\le
1-\alpha^{1/K}.
\end{equation}
At the 95\% confidence level, the corresponding upper bounds are 31.23\%, 17.07\%, 8.94\%, and 4.57\% for $K=8,16,32,$ and $64$, respectively. Larger rollout budgets therefore reduce finite-sampling uncertainty, but increase rollout generation and verification costs approximately linearly. We use $K=16$ during GRPO training as a practical compromise between boundary-estimation stability and computational cost.

We further evaluate the sensitivity of the empirical boundary using $K\in{1,2,4,8,16,32,64}$. For each budget, we report the proportion of prompts for which no verified solution is observed. We also compare the training budget $K=16$ with the higher-budget $K=64$ reference by reporting the proportion of $K=16$ beyond-boundary prompts that remain beyond at $K=64$, together with the proportion that recover to within-boundary status. The $K=64$ result is used only as a higher-budget empirical reference and should not be interpreted as an oracle label of intrinsic knowledge absence.

\begin{figure}[!t]
    \centering
    \includegraphics[width=0.75\textwidth]{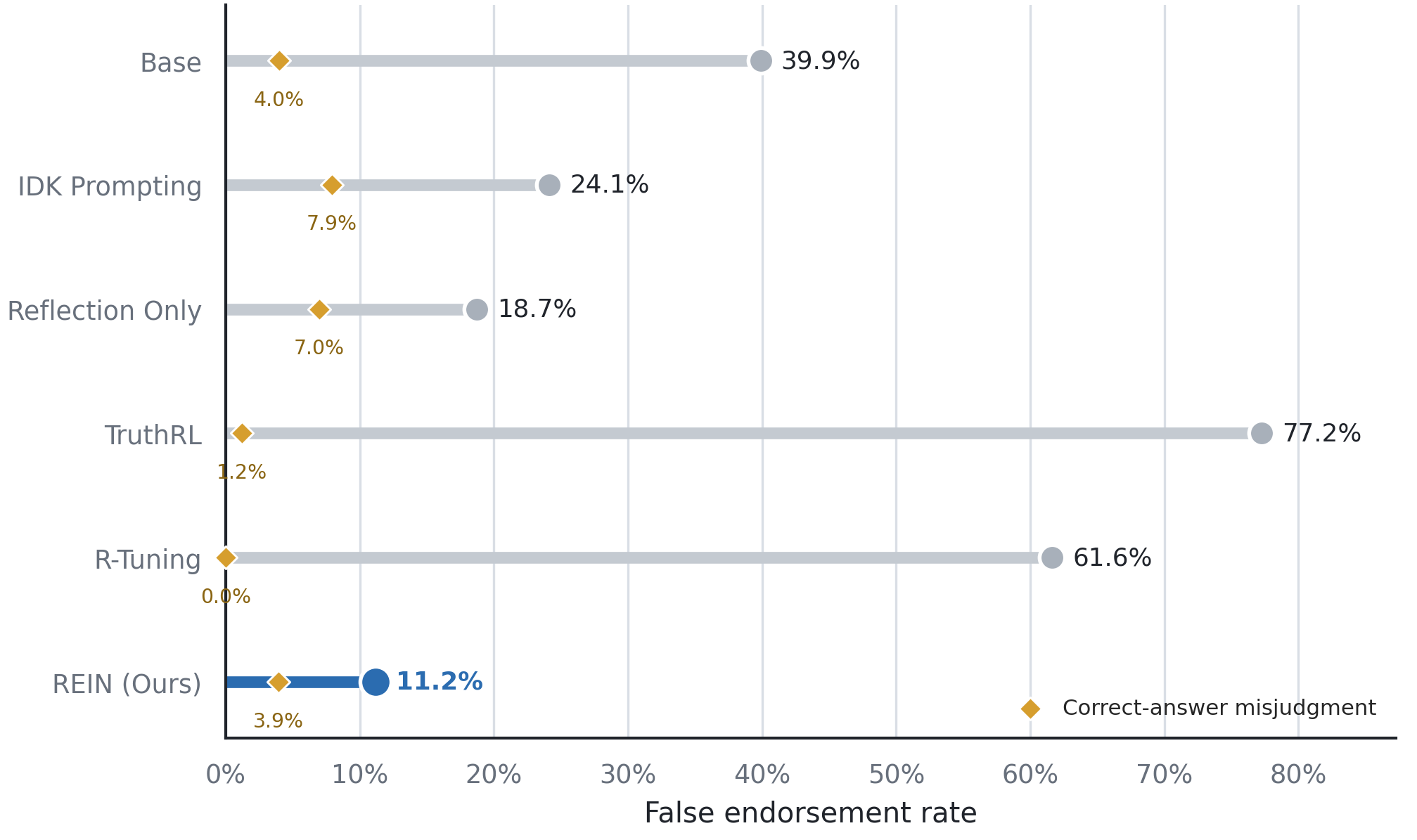}
    \caption{Reduction of false endorsement through reflection alignment. False endorsement measures the fraction of incorrect answers that are incorrectly judged as reliable. REIN substantially reduces false endorsements while maintaining a low rate of correct-answer misjudgment.}
    \label{fig:mcc_fig2}
\end{figure}

\subsection{Confidence Bound of the Group-level Boundary Indicator}
\label{app:boundary_confidence_bound}

The group-level boundary indicator $B(x)$ estimates whether a prompt lies
within the model's effective capability under the current policy and sampling
budget. Although $B(x)$ does not represent the true epistemic boundary of the
model, we show that observing no correct samples provides a probabilistic upper
bound on the probability of producing a correct answer.

\begin{lemma}[Confidence bound of $B(x)$]
\label{lem:boundary_confidence}
Assume that $K$ completions are independently sampled from the policy
\[
\{(c^{(k)},r^{(k)},y^{(k)})\}_{k=1}^{K}
\sim \pi_\theta(\cdot|x),
\]
where each completion has success probability
\[
p(x)=\Pr[y=y^\star(x)\mid x,\pi_\theta].
\]
If the group-level boundary indicator satisfies
\[
B(x)=\mathrm{beyond},
\]
meaning that none of the $K$ sampled completions produces the correct answer,
then, with confidence level $1-\alpha$,
\[
p(x)\leq 1-\alpha^{1/K}.
\]
\end{lemma}

\begin{proof}
For a single completion sampled from $\pi_\theta(\cdot|x)$, the probability of
obtaining an incorrect answer is
\[
1-p(x).
\]

Since the $K$ samples are independent, the probability that all sampled
completions fail to produce the correct answer is
\[
\Pr[B(x)=\mathrm{beyond}]
=
(1-p(x))^K.
\]

Observing $B(x)=\mathrm{beyond}$ means that zero successful samples are
obtained in $K$ independent trials. For a confidence level $1-\alpha$, the
failure probability satisfies
\[
(1-p(x))^K\geq \alpha .
\]

Taking the $K$-th root on both sides gives
\[
1-p(x)\geq \alpha^{1/K}.
\]

Therefore,
\[
p(x)\leq 1-\alpha^{1/K}.
\]

Hence, when no correct answer is observed among $K$ independent samples, the
probability that the model can correctly answer the prompt is bounded above by
$1-\alpha^{1/K}$ with confidence level $1-\alpha$.
\end{proof}

\paragraph{Example.}
For $K=16$ samples and a $90\%$ confidence level
($\alpha=0.10$), the upper bound becomes
\[
p(x)\leq 1-0.1^{1/16}\approx 0.134.
\]

Therefore, if none of the 16 sampled completions produces the correct answer,
we can infer with $90\%$ confidence that the probability of the current policy
solving the prompt is at most $13.4\%$.

\section{Implementation Details and Computational Resources}
\label{app:Implementation Details}

\begin{figure*}[!t]
    \centering
    \includegraphics[width=\textwidth]{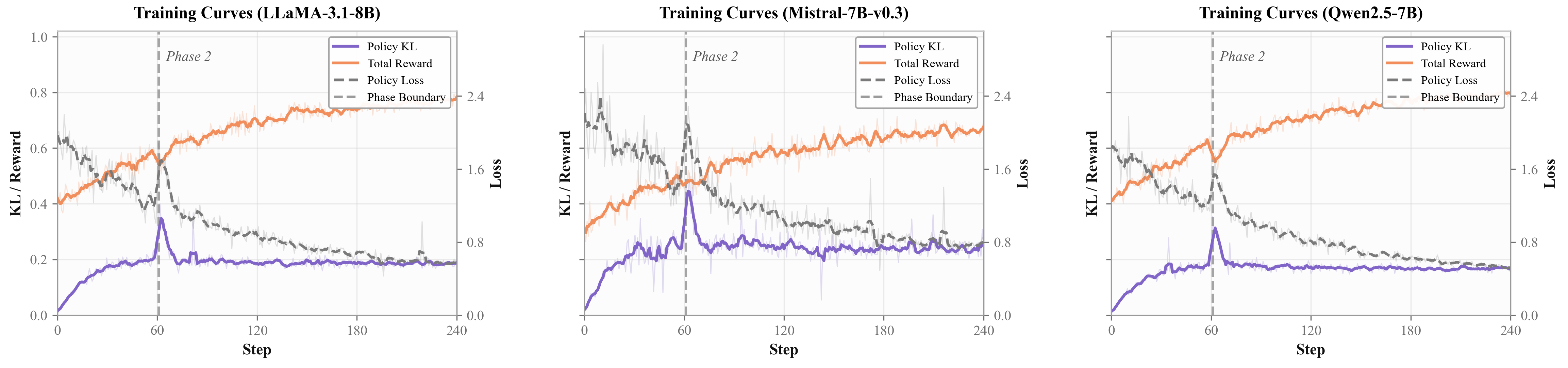}
    \caption{Training curves of KL divergence, reward, and loss. KL remains low and stable, while reward steadily increases and loss decreases throughout training.}
    \label{fig:training_kl}
\end{figure*}

\subsection{Implemented Variants and Scope}

The repository contains several implemented variants beyond the minimal baseline described above. These include 7B and 8B training configs, veracity-augmented reward variants, and additional prompt templates that increase the amount of reflective structure. We treat these as implementation variants rather than as fully reported analyses unless accompanied by committed experimental outputs.

\subsubsection{Reward Variants}

The primary distinction between the released GRPO recipes is whether the training objective uses three rewards \sloppy (\texttt{accuracy} + \texttt{xml\_structure} + \texttt{reflection\_presence})
or a four-reward variant that additionally includes
\texttt{reflection\_veracity}. This makes it possible to compare a simpler
format-and-correctness baseline against a variant that explicitly aligns
reflection judgment with draft $y_0$ correctness.

\subsubsection{Training and Evaluation Details}

The training config and hyperparameter are shown in table \ref{tab:training_config}.

\begin{table}[t]
\centering
\small
\renewcommand{\arraystretch}{1.05}
\resizebox{0.4\textwidth}{!}{%
\begin{tabular}{lc}
\toprule
\textbf{Item} & \textbf{Setting} \\
\midrule
GRPO samples ($K$) & 16 \\
Max length & 2048 \\
Temperature & 0.7 \\
LoRA $(r,\alpha)$ & $(16,32)$ \\
GRPO $\beta$ & 0.05 \\
\midrule
$r_{\mathrm{acc}}$ weight & 8.0 \\
$r_{\mathrm{xml}}$ weight & 3.0 \\
$r_{\mathrm{pres}}$ weight & 1.5 \\
$r_{\mathrm{vrcty}}$ weight & 2.0 \\
$r_{\mathrm{idk}}$ weight & 4.0 \\
\bottomrule
\end{tabular}%
}
\caption{We have a shared configuration across all the backbones unless we say otherwise. Other implementation details will adhere to the released configurations unless otherwise indicated.}
\label{tab:training_config}
\end{table}

As per the instructions provided in the main evaluation script, we use the single-pass generation decoding protocol with \texttt{do\_sample=False}. The main evaluation script supports the modes of reflection, instruction, and no-CoT. It produces the final answers using \texttt{boxed\{\}} for MATH/AIME-style evaluation and delimited by \texttt{\#\#\#\#} for GSM8K.
There is no use of test-time majority voting, self-consistency, or multi-round critique loops.

\subsubsection{Diagnostic Scope}

In addition to task-level metrics, the repository includes a diagnostic script
for evaluating whether reflections identify the erroneous step and error type on
a manually annotated subset of incorrect generations. We view this as a useful
analysis tool for future experiments, but not as a replacement for the main
benchmark metrics of accuracy, MCC, and hallucination rate.

\begin{figure*}[!t]
    \centering
    \includegraphics[width=\textwidth]{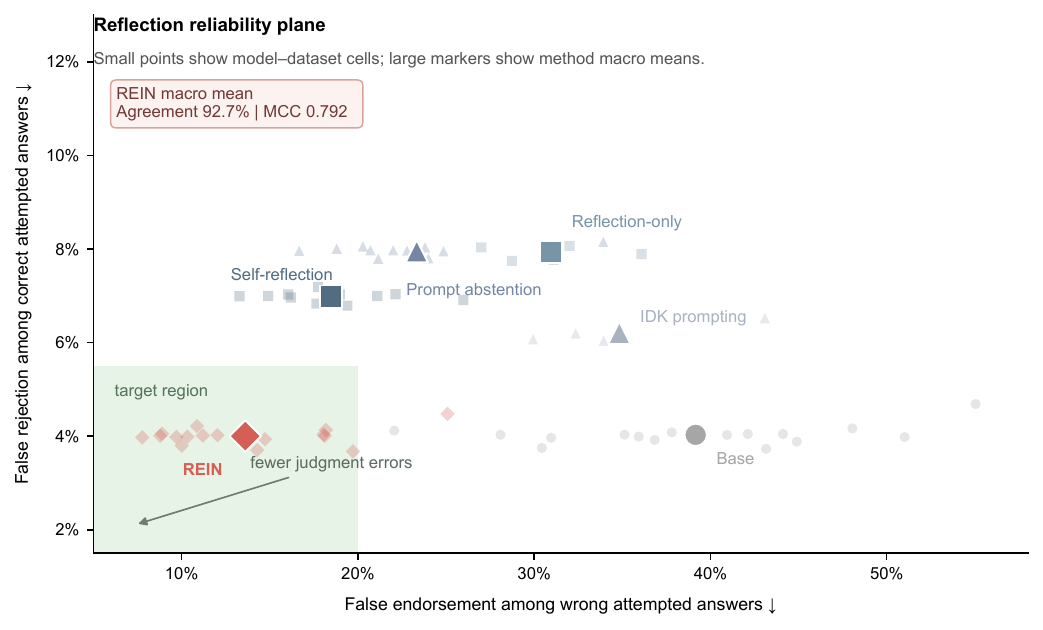}
    \caption{Reflection judgments align with answer outcomes across backbones and benchmarks. Each small point denotes one observed model–dataset cell, while large markers report method-level macro averages. False endorsement measures the fraction of wrong attempted answers judged reliable; false rejection measures the fraction of correct attempted answers judged unreliable. Lower values are better. Comparisons are descriptive because the available numbers of model–dataset cells differ across methods.}
    \label{fig:ref_reliability}
\end{figure*}

\subsection{Hardware and Duration}

We conduct our experiments on a computing system with NVIDIA RTX PRO™ 6000 Blackwell GPUs, each equipped with 96GB of VRAM. We trained the models on two GPUs and the average duration of one full training of a model with final learning rates and exhaustive evaluation around ~30h.

\section{Robustness of the Measurement}
\label{sec:robustness_measurement}

Every number in this paper passes through a prompt, a
decoding pass and a parser before it becomes a metric. Each
of those is a degree of freedom that could be tuned, after the
fact, in a direction that favours the proposed method. This
appendix fixes all three and measures how much the headline
results move when they are varied: across abstention surface
forms, across the output conflicts the parser must adjudicate,
across paraphrases of the evaluation prompt, and across four
parser implementations.

\subsection{Abstention Surface Forms}
\label{sec:abstention_surface_forms}

The abstention parser recognizes a fixed lexicon rather than
a single string, because a model that abstains in words the
parser does not know would be scored as producing an invalid
output and would lose coverage it has in fact earned.

Table~\ref{tab:abstention_surface_forms} shows the question is
close to moot for a trained model: 97.0\% of REIN's abstentions
use the canonical string the reward was written against, and
99.6\% fall in the top three forms. Restricting the parser to
the canonical string alone reclassifies the seven non-canonical
abstentions as invalid outputs. It cannot move coverage,
because Eq.~(1) makes abstentions and invalid outputs alike
non-substantive, and it cannot move selective accuracy or
H-Proxy, because neither outcome enters their numerators or
denominators; the lexicon governs only the abstention--invalid
split reported in Appendix~G. What the lexicon does affect is
the prompted baselines, whose abstentions are far less uniform,
and for which a canonical-only parser would misclassify
genuine refusals as malformed output. Going the other way
and instructing REIN to abstain in each of the eight forms in
turn leaves macro coverage in $[89.9,90.9]$ and H-Proxy in
$[10.9,11.8]$, so no reported conclusion depends on which
phrase the abstention is required to take. The four forms
the trained model never produces are in the lexicon for the
baselines alone.

\begin{table}[t]
\centering
\small
\resizebox{0.5\textwidth}{!}{%
\begin{tabular}{lcc}
\toprule
\textbf{Surface form}
& \textbf{Detect.}$\uparrow$
& \textbf{Emitted} \\
\midrule
\texttt{I don't know} (canonical)
& 100.0 & 97.0 \\
\texttt{I do not know}
& 100.0 & 1.7 \\
\texttt{\textbackslash boxed\{\textbackslash text\{I don't know\}\}}
& 99.8 & 0.8 \\
\texttt{I'm not sure}
& 99.6 & 0.4 \\
\texttt{Unknown}
& 99.1 & 0.0 \\
\texttt{Cannot be determined}
& 98.4 & 0.0 \\
\texttt{N/A}
& 97.2 & 0.0 \\
\texttt{IDK}
& 96.5 & 0.0 \\
\bottomrule
\end{tabular}%
}
\caption{Abstention surface forms, REIN on Qwen2.5-7B
pooled over the four benchmarks. Detect.\ is the parser's
recall on held-out completions manually labelled as
abstentions; Emitted is the share of REIN's 236 explicit
abstentions taking that form under the canonical prompt, so
the column is 229/4/2/1/0/0/0/0 expressed as percentages.
Both columns are percentages.}
\label{tab:abstention_surface_forms}
\end{table}

\subsection{Conflicting Outputs}
\label{sec:conflicting_outputs}

A completion can contradict itself, and how those cases are
resolved is a choice that must be declared rather than buried
in an implementation.

Self-contradiction is rare, 1.65\% of completions, and the
three adjudication rules of Table~\ref{tab:conflicting_outputs}
separate the headline metrics by at most 0.4 points. The
conservative default is not the one that flatters REIN: the
permissive rule would report a lower H-Proxy (10.9) and
higher coverage (91.0). We keep the conservative rule because
it is the only one of the three that cannot be accused of
resolving ambiguity in the proposed method's favour, and we
apply it identically to every baseline.

Class C5 deserves a note because it is the only conflict the
two judges of Appendix~A handle differently. Eight
completions carry an endorsing stance and an explicit
abstention: the model asserts its reasoning is sound and then
declines to commit to an answer. The native reading records
no substantive answer, so these instances leave the H-Proxy
denominator entirely; the post-hoc probe is never invoked,
since it is elicited only for substantive answers. They are
counted as abstentions throughout.

\begin{table}[t]
\centering
\begin{minipage}{0.5\textwidth}
\centering
\small
\renewcommand{\arraystretch}{1.3}
\setlength{\tabcolsep}{0.6pt}
\textbf{(a) Taxonomy and incidence}

\smallskip

\resizebox{\linewidth}{!}{%
\begin{tabular}{clcc}
\toprule
\textbf{Conflict}
& \textbf{Description}
& \textbf{Count}
& \textbf{\%} \\
\midrule
C1
& abstention marker inside a substantive answer
& 11 & 0.34 \\
C2
& multiple distinct answers in \texttt{<answer>}
& 7 & 0.22 \\
C3
& answer only in \texttt{<reflection>}
& 5 & 0.16 \\
C4
& stance absent or unparsable
& 19 & 0.59 \\
C5
& endorsing stance with explicit IDK
& 8 & 0.25 \\
C6
& duplicated or nested tag structure
& 3 & 0.09 \\
\midrule
& \textbf{total}
& \textbf{53}
& \textbf{1.65} \\
\bottomrule
\end{tabular}%
}

\medskip

\textbf{(b) Sensitivity to the adjudication rule}

\smallskip

\resizebox{\linewidth}{!}{%
\begin{tabular}{lcccc}
\toprule
\textbf{Rule}
& \textbf{Sel.Acc.}$\uparrow$
& \textbf{H-Proxy}$\downarrow$
& \textbf{Cov.}$\uparrow$
& \textbf{Eff.Acc.}$\uparrow$ \\
\midrule
Conservative (default)
& 78.6 & 11.3 & 90.4 & 71.5 \\
Permissive
& 78.9 & 10.9 & 91.0 & 72.0 \\
Drop conflicts
& 78.8 & 11.2 & 90.1 & 71.3 \\
\bottomrule
\end{tabular}%
}
\end{minipage}

\caption{Self-contradictory outputs from the REIN model on
Qwen2.5-7B, pooled over the four benchmarks
($N=3220$). The conservative rule treats C1--C3 and C6 as
invalid and C4 as a non-endorsement, which is the convention
used everywhere else in this paper; the permissive rule
resolves each conflict in the model's favour by taking the
first well-formed answer and reading C4 as an endorsement;
the third rule removes the 53 affected instances from all
denominators.}
\label{tab:conflicting_outputs}
\end{table}

\section{Qualitative Examples}

\begin{tcolorbox}[
    colback=gray!5,
    colframe=gray!40,
    boxrule=0.5pt,
    arc=2pt,
    left=6pt,
    right=6pt,
    top=6pt,
    bottom=6pt,
    title=\textbf{A representative case of correcting reasoning errors within the knowledge boundary},
    colbacktitle=gray!40,
    coltitle=black,
    fonttitle=\normalsize
]

\small

\noindent\textbf{System:} You are a careful reasoner. For every problem, reply with EXACTLY this structure:
\texttt{<think>...</think> <reflection>...</reflection> <answer>...</answer>}.

\smallskip

\noindent\textbf{User:} 
A store has 120 apples. They sell $\frac{1}{3}$ of the apples in the morning and $\frac{1}{4}$ of the remaining apples in the afternoon. How many apples are left?

\smallskip

\noindent\textbf{Assistant:}

\texttt{<think>} \\
Morning: $120 \times \frac{1}{3}=40$ sold, leaving 80 apples. 
Afternoon: $120 \times \frac{1}{4}=30$ sold, leaving $80-30=50$ apples.  \\
\texttt{</think>}
\texttt{<reflection>} \\

\textbf{Wait.} The afternoon sale should be $\frac{1}{4}$ of the remaining 80 apples, not the original 120.
Recalculating: $80 \times \frac{1}{4}=20$ sold in the afternoon.
Remaining apples: $80-20=60$.
The initial calculation contains an error in step 2. \\
\texttt{</reflection>}
\texttt{<answer>}\\
\texttt{60} \\
\texttt{</answer>}

\end{tcolorbox}

\begin{tcolorbox}[
    colback=gray!5,
    colframe=gray!40,
    boxrule=0.5pt,
    arc=2pt,
    left=6pt,
    right=6pt,
    top=6pt,
    bottom=6pt,
    title=\textbf{A representative case of the REIN framework triggering selective abstention},
    colbacktitle=gray!40,
    coltitle=black,
    fonttitle=\normalsize
]

\small
\noindent\textbf{System:} You are a careful reasoner. For every problem, reply with EXACTLY this structure: \texttt{<think>...</think> <reflection>...</reflection> <answer>...</answer>}.

\smallskip
\noindent\textbf{User:} Find the unique 3-digit integer $N$ such that $N$ is twice the sum of its digits cubed.

\smallskip
\noindent\textbf{Assistant:}\\
\texttt{<think>} \\
Such Armstrong-like number is uncommon. I would need to search digit-by-digit, but this is computationally heavy without a calculator. \\
\texttt{</think>}
\texttt{<reflection>} \textbf{Unsure.} The search space requires careful enumeration that I cannot do reliably by hand. \texttt{</reflection>}
\texttt{<answer>} \\
\texttt{I don't know} \\
\texttt{</answer>}

\end{tcolorbox}

\begin{tcolorbox}[
    colback=gray!5,
    colframe=gray!40,
    boxrule=0.5pt,
    arc=2pt,
    left=6pt,
    right=6pt,
    top=6pt,
    bottom=6pt,
    title=\textbf{An empirical failure mode of the self-correction mechanism within the REIN framework},
    colbacktitle=gray!40,
    coltitle=black,
    fonttitle=\normalsize
]

\small
\small
\noindent\textbf{System:} You are a careful reasoner. For every problem, reply with EXACTLY this structure:
\texttt{<think>...</think> <reflection>...</reflection> <answer>...</answer>}.

\smallskip
\noindent\textbf{User:} Solve: $x^2 - 4 = 0$.

\smallskip
\noindent\textbf{Assistant:}\\
\texttt{<think>} \\
$x = \pm 2$. Recorded: 2. \\
\texttt{</think>}
\texttt{<reflection>} \\
\textbf{Wrong.} There are two roots $\pm 2$; I only listed one. \\ \texttt{</reflection>}
\texttt{<answer>} \\
\texttt{\textbackslash boxed\{2\}} \\
\texttt{</answer>}

\end{tcolorbox}

\section{Implementation Details for Baselines.}
To ensure a fair comparison and reproducibility, all trainable
baselines are implemented with the same backbone models,
training data, and parameter-efficient fine-tuning strategy as
REIN. Specifically, we use LoRA-based fine-tuning and keep
the optimization configurations consistent whenever applicable.
Each baseline is trained using its original objective, while
the inference-time decoding and evaluation protocol are unified
across all methods. The implementation details of each baseline
are described as follows:

\begin{itemize}

\item \textbf{IDK Prompting.}
IDK Prompting is an inference-time baseline that introduces
explicit abstention instructions into the prompt. The model is
encouraged to respond with ``I don't know'' when it is uncertain
or when the problem is beyond its capability. No additional
training is performed, and the original model parameters remain
unchanged.

\item \textbf{Reflection-only.}
Reflection-only preserves the reflection format and reflection
generation process used in REIN, allowing the model to assess
the reliability of its own reasoning. However, it removes the
boundary-aware IDK component and does not include the
knowledge-boundary-based abstention reward. Therefore, this
baseline evaluates the contribution of reflection alignment
alone without selective abstention learning.

\item \textbf{R-Tuning.}
We reproduce R-Tuning following its original uncertainty-aware
instruction tuning objective. The model is trained to explicitly
produce ``I don't know'' responses on questions beyond its
capability while preserving answer generation on solvable
questions.

\item \textbf{TruthRL.}
We adopt the same GRPO optimization framework and training
configuration as REIN to ensure a controlled comparison. The
only difference is the reward formulation. Following the
original TruthRL design, the reward assigns positive feedback
to correct answers, neutral reward to explicit abstention, and
negative reward to incorrect answers.

\end{itemize}

\end{document}